\documentclass[sigconf]{aamas}  % do not change this line!
\usepackage{balance}
\usepackage{amsmath}
\usepackage{nicefrac}
\usepackage{todonotes}
\usepackage{algorithm}
\usepackage[noend]{algpseudocode}
\usepackage{bm}
\usepackage{float}
\usepackage{tikz}
\usetikzlibrary{arrows,positioning}

\usepackage{dblfloatfix}

\usepackage{lipsum}  

\usepackage{graphicx}
\usepackage{subfigure}
\tikzset{
    >=stealth',
    punkt/.style={
           rectangle,
           rounded corners,
           draw=black, very thick,
           text width=9em,
           minimum height=1.7em,
           text centered},
    pil/.style={
           ->,
           thick,
           shorten <=2pt,
           shorten >=2pt,}
}

\setcopyright{ifaamas}  % do not change this line!
\acmDOI{}  % do not change this line!
\acmISBN{}  % do not change this line!
\acmConference[AAMAS'19]{Proc.\@ of the 18th International Conference on Autonomous Agents and Multiagent Systems (AAMAS 2019)}{May 13--17, 2019}{Montreal, Canada}{N.~Agmon, M.~E.~Taylor, E.~Elkind, M.~Veloso (eds.)}  % do not change this line!
\acmYear{2019}  % do not change this line!
\copyrightyear{2019}  % do not change this line!
\acmPrice{}  % do not change this line!

\DeclareMathOperator*{\argmax}{arg\,max}
\DeclareMathOperator*{\argmin}{arg\,min}

\newcommand{\myitemA}{\item[(\theenumi$\vphantom{x}{a}$)]}

\newcommand{\itembf}{\item[(\stepcounter{enumi}\textbf{\arabic{enumi}})]}

\DeclareRobustCommand{\VAN}[3]{#2}
\DeclareRobustCommand{\VON}[3]{#2}
\DeclareRobustCommand{\DE}[3]{#2}

\newcommand{\ko}{\kappa}
\newcommand{\kp}{\varkappa}

\newcommand{\note}[2][Note]{\textcolor{blue}{\textbf{\large [}\colorbox{yellow}{\textbf{#1:}}{\small #2}\textbf{\large ]}}}
\renewcommand{\note}[2][]{}

\newcommand{\vsp}[1]{\vspace{#1}}
\renewcommand{\vsp}[1]{}
\begin{document}

\title{Robust Temporal Difference Learning for Critical Domains}  % put your title here!
%\titlenote{Produces the permission block, and copyright information}

% AAMAS: as appropriate, uncomment one subtitle line; check the CFP
%\subtitle{Extended Abstract}
%\subtitle{Industrial Applications Track}
%\subtitle{Socially Interactive Agents Track}
%\subtitle{Blue Sky Ideas Track}
%\subtitle{Engineering Multiagent Systems Track}
%\subtitle{Robotics Track}
%\subtitle{JAAMAS Track}
%\subtitle{Doctoral Mentoring Program}

%\subtitlenote{The full version of the author's guide is available as \texttt{acmart.pdf} document}

% AAMAS: submissions are anonymous for most tracks
% \author{Paper \#659}  % put your paper number here!

%% example of author block for camera ready version of accepted papers: don't use for anonymous submissions
%
\author{Richard Klima}
%\authornote{Dr.~Trovato insisted his name be first.}
%\orcid{1234-5678-9012}
\affiliation{%
  \institution{University of Liverpool}
%  \streetaddress{P.O. Box 1212}
  \city{Liverpool} 
%  \state{Ohio}
  \country{United Kingdom}
%  \postcode{43017-6221}
}
\email{richard.klima@liverpool.ac.uk}
\author{Daan Bloembergen}
%\authornote{The secretary disavows any knowledge of this author's actions.}
\affiliation{%
  \institution{Centrum Wiskunde \& Informatica}
%  \streetaddress{P.O. Box 1212}
  \city{Amsterdam}
  \country{The Netherlands}
%  \state{Ohio} 
%  \postcode{43017-6221}
}
\email{d.bloembergen@cwi.nl}
\author{Michael Kaisers}
%\authornote{The secretary disavows any knowledge of this author's actions.}
\affiliation{%
  \institution{Centrum Wiskunde \& Informatica}
%  \streetaddress{P.O. Box 1212}
  \city{Amsterdam}
  \country{The Netherlands}
%  \state{Ohio} 
%  \postcode{43017-6221}
}
\email{m.kaisers@cwi.nl}
\author{Karl Tuyls}
%\authornote{This author is the%%%% LET'S USE THIS FOR OUR OPERATOR %%%%
%  one who did all the really hard work.}
\affiliation{%
  \institution{University of Liverpool}
%  \streetaddress{1 Th{\o}rv{\"a}ld Circle}
  \city{Liverpool} 
  \country{United Kingdom}
}
\email{karltuyls@google.com}
%
%\author{Charles Palmer}
%\affiliation{%
%  \institution{Palmer Research Laboratories}
%  \streetaddress{8600 Datapoint Drive}
%  \city{San Antonio}
%  \state{Texas} 
%  \postcode{78229}}
%\email{cpalmer@prl.com}
%
%\author{John Smith}
%\affiliation{\institution{The Th{\o}rv{\"a}ld Group}}
%\email{jsmith@affiliation.org}
%
%\author{Julius P.~Kumquat}
%\affiliation{\institution{The Kumquat Consortium}}
%\email{jpkumquat@consortium.net}
%
%% The example's default list of authors is too long for headers
%\renewcommand{\shortauthors}{B. Trovato et al.}

\begin{abstract}  % put your abstract here!
We present a new Q-function operator for temporal difference (TD) learning methods that explicitly encodes robustness against significant rare events (SRE) in critical domains. The operator, which we call the $\ko$-operator, allows to learn a robust policy in a model-based fashion without actually observing the SRE. We introduce single- and multi-agent robust TD methods using the operator $\ko$. We prove convergence of the operator to the optimal robust Q-function with respect to the model using the theory of Generalized Markov Decision Processes. In addition we prove convergence to the optimal Q-function of the original MDP given that the probability of SREs vanishes. Empirical evaluations demonstrate the superior performance of $\ko$-based TD methods both in the early learning phase as well as in the final converged stage. In addition we show robustness of the proposed method to small model errors, as well as its applicability in a multi-agent context.
\end{abstract}

  % put your semicolon-separated keywords here!

\maketitle

%%%%%%%%%%%%%%%%%%%%%%%%%%%%%%%%%%%%%%%%%%%%%%%%%%%%%%%%%%%%%%%%%%%%%%%%%%%%%%%%%%%%%%%%%%%%%%%%%%%%%%%%%
%% start of main body of paper
\vsp{-1em}
\section{Introduction}
% \note{better define and distinguish ``safe'' and ``robust''}

Many critical systems exhibit global system dynamics that are highly sensitive to the local performance of individual components. This holds for example for (air) traffic and transport networks, communication networks, security systems, and (smart) power grids \citep{cristian1996fault,shooman2003reliability,knight2002safety,liu2012cyber}. In each case, the failure of or malicious attack on a small set of nodes may lead to knock-on effects that can potentially destabilise the whole system. Innovations in critical systems may introduce additional vulnerabilities to such attacks: e.g., in smart grids communication channels are needed for distributed intelligent energy management strategies, while simultaneously forming a potential target that could compromise safety~\citep{yan2013survey}. Our research is motivated precisely by the need for safety in these critical systems, which can be achieved by building in robustness against rare but significant deviations caused by one or more system components failing or being compromised in an attack.

In this article we present a new approach %\note[Karl]{TO MAINTAIN ANONYMITY I AM NOT SURE WE SHOULD KEEP THE FOOTNOTE}\footnote{Parts of this work have been recently presented in workshops without official proceedings.} 
for learning policies in such systems that are robust against a chosen scenario of potential attacks or failures. We accomplish this by introducing a new Q-function operator, which we call the $\ko$-operator, that encodes robustness into the bootstrapping update of traditional temporal difference (TD) learning methods. In particular, we design the operator to encode the possibility of significant rare events (SREs) without requiring the learning agent to observe such events in training. Although the $\ko$-operator is model-based with respect to these SREs, it can be combined with any TD method and can thus still be model-free with respect to the environment dynamics.

We prove convergence of our methods to the optimal robust Q-function with respect to the model using the theory of Generalized Markov Decision Processes. In addition we prove convergence to the optimal Q-function of the original MDP given that the probability of SREs vanishes. Empirical evaluations demonstrate the superior performance of $\ko$-based TD methods both in the early learning phase as well as in the final converged stage. In addition we show robustness of the proposed method to small model errors, as well as its applicability to multi-agent joint-action learning.

% \note[Daan]{This last section should be tailored to the actual content. Can we say something as well about variance reduction? And convergence speed / sample complexity?}
%The remainder of this article is structured as follows: The next section embeds our approach in related work, followed in Section~\ref{sec:background} by an introduction of background concepts. Section~\ref{sec:operator} introduces the new TD operator $\ko$, for which we subsequently prove convergence. Section~\ref{sec:experiments} provides empirical results, and Section~\ref{sec:discussion} concludes.

This articles proceeds with related work in Section~\ref{sec:related} and background concepts in Section~\ref{sec:background}. We formally introduce the new TD operator $\ko$ in Section~\ref{sec:operator} and subsequently prove convergence. Section~\ref{sec:experiments} provides empirical results, and Section~\ref{sec:discussion} concludes. 

\vsp{-0.5em}
\section{Related Work} \label{sec:related}
% \note[Richard]{Trim this section down?}

% \note[Karl]{General remark on RW: we need a better flow, and we need to be more specific here and there.}
% \note[Daan]{We could also link to the research area of adversarial machine learning here.}
The aim to find robust policies is relevant to multiple research areas, including security games, robust control/learning, safe reinforcement learning and multi-agent reinforcement learning.

% security games
The domain of \textbf{security games} has expanded in recent years with many real-world applications in critical domains~\citep{Pita2008, Shieh2012}, where the main approach has been computing exact solutions and deriving strong theoretical guarantees, mostly using equilibria concepts such as Nash and Stackelberg equilibria~\citep{Korzhyk2011a, lou2017}.
In contrast, we base our approach on \textit{reinforcement learning from interactions} with the environment, thus we do not need to know the system model; such an approach to security games has been studied less,
% Until now there has been substantially less work on reinforcement learning for security games than on game-theoretic approaches, 
exceptions being for example \citet{ruan2005} and \citet{klima2016} who use reinforcement learning in the context of patrolling and illegal rhino poaching problems, respectively.
Security games often assume frequent adversarial attack, whereas our work focuses on occasional loss of control over the system, which can represent e.g. failures or adversarial attack.
% and are defined on the system action space.
% This has been very important to theoretically underpin the field, however it often seems difficult to deploy exact theoretical solution methods in real world settings due to strict model assumptions or severe simplifications~\cite{tambe2012}. 
% On the one hand, 
Moreover, our work adopts the information asymmetry assumption often used in Stackelberg Security games \citep{Korzhyk2011a}, providing the model of attack types for the \emph{leader}, and allowing leader-strategy-informed best response strategies by attackers.
% On the other, 
%
% robust control
Similar to security games, control theory starts with a model of the system to be controlled (the \emph{plant}), and for the purpose of \textbf{robust control} assumes a set of possible plants as an explicit model of uncertainty, seeking to design a policy that stabilises all these plants~\citep{zhou1998essentials}.
A slightly weaker assumption is made in related work that assumes control over the number of observations for \emph{significant rare events} (SREs), performing updates by sampling from the model~\citep{ciosek2017}. 
% It introduces a policy gradient variant to improve learning in presence of SREs, using a \emph{proposal distribution} that controls data from which it learns and \emph{importance sampling} to adjust updates.
In contrast, our work assumes that the model of this system is not known a priori, and a policy needs to be \textit{learned by interacting} with it.
While early work on \textbf{robust reinforcement learning} focused on learning within parameterised acceptable policies~\citep{singh1994robust}, later work transferred the objective of maximising tolerable disturbances from control theory to reinforcement learning~\citep{morimoto2005robust}.
Our work is similar to the therein defined \emph{Actor-disturber-critic}, but we replace its model of minimax simultaneous actions with stochastic transitions between multiple controllers (one being in control at any time) with arbitrary objectives for each controller.
%
% \note[Karl]{I find the next statement actually very vague; we are in between 2 approaches, ok, but it depends on the chosen controller objectives; hmmm, what does this actually mean?}
% \note[Richard]{Include the following references}
In relation to the taxonomy of \textbf{safe reinforcement learning} of \citet{garcia2015comprehensive} our method falls in between \emph{Worst-Case Criterion under Parameter Uncertainty} and \emph{Risk-Sensitive Reinforcement Learning Based on the Weighted Sum of Return and Risk}, depending on the chosen alternate controller objectives. 
Our Q($\ko$) method is comparable to the $\beta$-pessimistic Q-learning method of \citet{gaskett2003}, however, we propose a more general $\ko$ operator of which Q($\ko$) is only an example. %, which can be used with other TD methods and in multi-agent settings.
Finally, our approach has commonalities with the \textbf{multi-agent reinforcement learning} algorithm Minimax-Q~\citep{littman1994} for zero-sum games, which assumes minimisation over the opponent action space. However, in contrast, we define an attack to minimise over our own action space, and thus learn (but not enact) simultaneously our optimal policy and the (rare) attacks it is susceptible to.
%, whereas Minimax-Q works with simultaneous minimisation
%In order to demonstrate the proposed approach in independent learning scenario we use Fictitious play~\cite{Fudenberg1998} as the simplest estimation of the cooperating agent behaviour.
We further cover not only minimising adversaries but also random failures or any other policy encoding other adversaries' agendas (see Section~\ref{sec:formal_model}).
%While \emph{on-policy} learning algorithms have been shown to perform better than classical Q-learning in perturbed environment~\cite{singh2000convergence}, and can thus in some sense be considered more robust, our method combines with both on- and off-policy learning, and provides robustness against a chosen target.

% The whole idea of using learning is primarily based on modelling the system as a Markov Decision Process (MDP) with states, actions and rewards. In case of considering multiple agents, MDPs extend to Stochastic Games. 
% Some approaches have used reinforcement learning for patrolling problems~\citep{ruan2005} or illegal rhino poaching problem~\citep{klima2016}. 
% Security games are often modelled as Stackelberg Security games \citep{Korzhyk2011a}, which capture the asymmetry in agents' information about the game. The attacker is assumed to observe the defender's past actions and thus can reason about his strategy, to which the attacker in turn best responds. Related work suggests a framework for asymmetric (Stackelberg) multi-agent reinforcement learning where the roles of defender (leader) and attacker (follower) are fixed~\citep{kononen2004}.

% \note[Karl]{Why is this a 'however'? I don't understand this sentence.}
% Building on that, we arrive at a more general safe learning approach where the attacks might be very rare and not necessarily adversarial (e.g. technical failures).

% \todo[inline]{cite~\citep{ciosek2017} - \textit{significant rare events} (SRE) and cite~\citep{foerster2017}.}

\section{Background}\label{sec:background}
\note[Richard]{Reviewer 1 says to elaborate more, make it more self-contained. I think it's not necessary and would leave it how it is.}

This work belongs to the field of reinforcement learning (RL)~\cite{sutton2018reinforcement}, and makes use of the core concept of a \textit{Markov decision process (MDP)}. An MDP is formally defined by a tuple $(S, A, R, P)$, where $S$ is a finite set of states, $A$ is a finite set of actions, $R(s, a) \rightarrow r \in \mathbb{R}$ is the reward function for a given state $s \in S$ and an action $a \in A$ and $P(s'|s, a)$ is the transition function giving a probability of reaching state $s'$ after taking action $a$ in state $s$. In this work we also consider a multi-agent setting, which uses the formalism of the \textit{stochastic game}, which generalizes the MDP to multiple agents and is defined by a tuple $(n, S, A_1 \ldots A_n, R_1 \ldots R_n, P)$, where $n$ is the number of agents and $A_i$ is the action space of agent $i$. The joint action space is $A = A_1 \cup \ldots \cup A_n$, and a joint action is $\bm{a}=(a_1, a_2, \ldots, a_n)$.\footnote{We use the common shorthand $\bm{a_{-i}}$ to denote the joint action of all agents except agent $i$, i.e., $\bm{a_{-i}}=(a_1,\ldots,a_{i-1},a_{i+1},\ldots,a_n)$.}
$R_i(s, a_i, \bm{a_{-i}}) \rightarrow r_i$ is the reward function of agent $i$ for given state $s$ and joint action $\bm{a}$, and $P(s'|s, \bm{a})$ is the state transition function.

The main goal of RL is finding an optimal policy for given MDP. A common method is \textit{temporal difference (TD)} learning, which estimates the value of a state by bootstrapping from the value-estimates of successor states using Bellman-style equations. 
%which is a one-step bootstrapping method based on Bellman style equations. 
%Of crucial importance in TD methods is the definition of the \textit{TD error}, which describes the difference between the current estimate of the (state-)value and a new sample obtained from interacting with the environment. The TD method updates its value estimate in order to reduce this error. 
TD methods work by updating their state-value estimates in order to reduce the \textit{TD error}, which describes the difference between the current estimate of the (state-)value and a new sample obtained from interacting with the environment.
In this work we focus on modifying the update \textit{target} of this TD error, which has the standard form of $r + \gamma V(s')$, where $\gamma$ is the discount factor and $V(s')$ is the current estimate of the next state's value. In \textit{on-policy} methods such as SARSA the target is induced by the actual (behaviour) policy being followed, while \textit{off-policy} methods use an alternative operator (e.g., greedy maximization as in Q-learning). We refer the reader to~\citet{sutton2018reinforcement} for an overview of RL.

\section{The Robust TD Operator $\ko$}\label{sec:operator}

%\note[Richard]{The whole section 4 needs to be redone, now it's mostly take from EWRL, we should focus on the adversarial attack.}

Before we formally define our robust TD operator $\ko$, we give an intuitive example. Suppose a Q-learning agent needs to learn a robust policy against a potential malicious adversary who could, with some probability $\kp$, take over control in the next state.\footnote{We use the symbol $\ko$ to denote the proposed TD operator and the symbol $\kp$ for the parameter denoting the probability of attack.} The value of the next state $s_{t+1}$ thus depends on who is in control: if the agent is in control, she can choose an optimal action that maximizes expected return; or if the adversary is in control he might, in the worst case, aim to minimize the expected return. This can be captured by the following modified TD error
\begin{equation*}
    \delta_t = r_{t+1} + \gamma \left((1-\kp)\max_{a} Q(s_{t+1},a) + \kp \min_{a} Q(s_{t+1},a)\right) - Q(s_t,a_t),
\end{equation*}
where we assume that the agent has knowledge of (or can estimate) the probability $\kp$.\footnote{Note that while a token of control could be included in the state (doubling its size), our approach instead directly applies model-based bootstrap updates. This makes it explicit that the robustness target is a chosen parameter of the operator, and allows to learn robust strategies before observing SREs or when learning during the SREs is not possible. This also highlights the difference between state transition probabilities, which are part of the environment and thus external to the agent, and the expected probability of SREs given by $\kp$ which are part of the agent's internal model.}

In the following we first present a formal, general model of the operator $\ko$, by modifying the target in the standard Bellman style value function. We then present practical implementations of TD($\ko$) methods that use this operator, for both single- and multi-agent settings, based on the classical on- and off-policy TD learning algorithms (Expected) SARSA and Q-learning. 
% Finally, we present several possible extensions to our model.

% \note[Daan]{Writing the previous paragraphs made me more convinced that we need different symbols for the operator and the parameter.}

% \note[Daan]{This is probably also the right place to formally make that distinction clear once we've decided on it.}

\subsection{Formal Model}\label{sec:formal_model}

% \note[Daan]{We need a formal definition of $\ko$ as an actual operator.}

We consider a set of $m$ possible \textit{control policies} $C=\{\sigma_1,\ldots,\sigma_m\}$. At each time step, one of these policies is in control (and thus decides on the next action) with some probability $p(\sigma_i|s)$ that may depend on the state $s$. %\footnote{Note that state transitions are part of the environment and thus external to the agent, whereas transitions caused by SREs may depend on the current value function (internal state) of the agent.} 
The set $C$ and probability function $p(\cdot)$ are assumed to be (approximately) known by the agent. In our new TD methods, the value of the next state $s'$ then becomes a function of both the state and the function $p(\cdot)$, which we capture in our proposed operator $\ko$ as $V^\ko(s')$. 
Note that the set $C$ includes the focal policy $\pi$ that we seek to optimise in face of (possibly adversarial) alternative controllers. Such external control policies can represent for example a malicious attacker, aiming to minimize the expected return, or any arbitrary dynamics, such as random failures (e.g. represented by a uniformly random policy).
Based on a prior assumption about the nature of $\sigma$ we want to optimise the focal policy $\pi$ without necessarily observing actual attacks or failures. This means learning our robust policy $\pi$ right from the start. 

% In the basic model discussed here w
We define $\sigma$ in terms of our own Q-value function, for example an attacker that is minimising our expected return. Thus we need to learn only one Q-value function $Q^\pi$. This is similar to the standard assumption in Stackelberg games that the attacker is able to fully observe our past actions and thus can enact the informed best response.
We define the Q-value function update for our policy $\pi$ based on standard Bellman equation and given the operator $\ko$ as
\begin{equation} \label{eq:general_model}
    Q^\pi(s,a) \leftarrow Q^\pi(s,a) + \alpha \Big[\underbrace{r + \gamma V^{\ko}({s}')}_\textrm{target} - Q^\pi(s,a) \Big].
\end{equation}
%
% %
% where the value function in state $s$, given the control transition function $\mathcal{C}$, is defined as
% \[
% \]
% with the state value function  $V^{\tilde{\pi}}({s'},\mathcal{C}) = \sum_{c_\sigma \in C} p(c_\sigma) V^\pi(s', \sigma)$ computed as a weighted sum over all possible policies $\sigma$. The composite policy $\tilde{\pi}(s,\bm{a}) = \sum_{c_\sigma \in C} p(c_\sigma) \sigma(s, \bm{a})$ can be seen as the actual expected policy that is executed by the system, including any potential malfunctions or/and attacks as given in the set of possible control domains $C$. Note that we can learn $Q^\pi$ without actually observing any attack or malfunction. 
% In Appendix~\ref{appendix:advanced_model} we also discuss a generalisation of the proposed method to a more advanced model in which the not-in-control policy $\mu$ is derived from a different Q-value function $Q^\mu$.
Note that where in the standard Bellman equation we would have $V^\pi(s) = \sum_{a} \pi(s,a)Q^\pi(s,a)$, in our case we have
\begin{equation}
    V^{\ko}(s) = \sum_{\sigma \in C} p(\sigma | s) \sum_{a} \sigma(s,a)Q^\pi(s,a),
    \label{eq:value_fn}
\end{equation}
%
% $V^{\tilde{\pi}}(s, \mathcal{C}) = \sum_a \sum_{c_\sigma \in \mathcal{C}} p(c_\sigma)\sigma(s,a)Q^\pi(s,a)$
computed as a weighted sum over all possible control policies $\sigma \in C$. %The composite policy $\tilde{\pi}$ represents the actual (stochastic) policy that is executed by the system, including any potential malfunctions and/or attacks. 
Note that we can learn $Q^\pi$ without actually experiencing any attack or malfunction, based only on prior assumptions about the possible control policies as captured by the operator $\ko$.
We refer to this target modification as the operator $\ko$ because it closely resembles the Bellman optimality operator $\mathcal{T}^\star$, which is defined as $\mathcal{T}^\star V (s) = \max_a \big [ R(s, a) + \sum_{s'} P(s'|s, a)\gamma V(s') \big]$. Thus, we can then formally define the $\ko$ optimality operator 
$\mathcal{T}^\star_\ko$ by substituting the value function $V(\cdot)$ with $V^\ko(\cdot)$.
% \note[Richard]{Not too sure how to write down the definition of the operator $\ko$ and whether we should even do it.}
% \begin{equation*}
% \resizebox{1\hsize}{!}{$
%     \mathcal{T}^\star_\ko V (s) = \max_a \Big[ R(s, a) + \sum_{s'} P(s'|s, a)\sum_{\sigma \in C} p(\sigma | s') \sum_{a'} \sigma(s',a') \gamma V(s') \Big].$}
% \end{equation*}
% \begin{equation*}
%     \mathcal{T}^\star_\ko V (s) = \max_a \Big[ R(s, a) + \sum_{s'} P(s'|s, a) \gamma V^\ko(s') \Big].
% \end{equation*}

In the following we present several $\ko$-versions of classical TD methods. For simplicity we assume a scenario in which we have only a single adversarial external policy $\sigma$ that aims to minimize our value, and thus $C=\{\pi,\sigma\}$. Note however that our model is general, and would work for any $C$ and $p(\cdot)$.

\subsection{Examples of TD($\ko$) Methods}\label{sec:td-algorithms}

% \note[Daan]{I removed the generalisation to $n$ agents to save space. Also, we don't really need it.}

We first present single-agent $\ko$-based learning methods by building on the standard TD methods Q-learning and Expected SARSA. Then we present two-agent joint-action learning approaches. Although a generalization to $n$ agents is relatively straightforward, we choose to focus solely on the single- and two-agent case in this paper for clarity of exposition.
In each case, we consider the setting in which either the focal agent, with policy $\pi$, is in control, or the external adversary with policy $\sigma$ aiming to minimize return. We further simplify the model by making the control policy probability function $p(\cdot)$ state-independent, reducing it to a probability vector.

\vsp{-0.5em}
\subsubsection{Single-Agent Methods}
\label{sec:single-agent-kappa}

Before we present the algorithms, it is important to note that we need to distinguish the \textit{target} and \textit{behaviour} policies. The $\ko$-operator is defined on the target (see Eq.~\eqref{eq:general_model}), while the behaviour policy is used only for selecting actions. We assume an $\epsilon$-greedy behaviour policy throughout.

In \textbf{off-policy Q($\bm{\ko}$)}, the target policy is the greedy policy $\pi(s) = \argmax_a Q(s,a)$ that maximizes expected return. The adversarial policy on the other hand aims to minimize the return, i.e., $\sigma(s) = \argmin_a Q(s,a)$. Assuming a probability of attack of $\kp$ as before, we have $p(\pi) = (1 - \kp)$ and $p(\sigma) = \kp$. Thus, Eq.~\eqref{eq:value_fn} becomes
\begin{equation*}
    V^\ko(s) = (1-\kp)\max_a Q(s,a) + \kp \min_a Q(s,a).
\end{equation*}

For \textbf{on-policy Expected SARSA($\bm{\ko}$)} the target is the (expectation over the) focal policy $\pi$, while the adversarial policy $\sigma$ remains the same as before. Thus, we have
\begin{equation*}
\begin{aligned}
   V^\ko(s)
   &= (1- \kp) \mathbb{E}_{a \sim \pi} \big[ Q(s,a) \big] + \kp  \min_a Q(s, a) \\
   &= (1- \kp) \sum_a \pi(a|s) Q(s,a) + \kp  \min_a Q(s, a).
\end{aligned}
\end{equation*}

\vsp{-0.5em}
\subsubsection{Multi-Agent Methods}\label{sec:multi-agent-kappa}

We move from a single-agent setting to a scenario in which multiple agents interact. 
For sake of exposition we only present a two-agent case with different action spaces, $A_1$ and $A_2$, but an identical reward function and thus a shared joint action Q-value function $Q: S \times A_1 \times A_2 \rightarrow \mathbb{R}$. %Moreover, the agents can coordinate a priori on a joint action to take in the current state, and thus simply find the joint action with the maximum Q-value. 
Moreover, we assume full communication during the learning phase, allowing the agents to take each other's policies into account when selecting the next action.\footnote{A common practice in cooperative multi-agent learning settings, see e.g.,~\citep{foerster2018counterfactual,sunehag2017value}.} Our algorithms are therefore based on the joint-action learning (JAL) paradigm~\cite{claus1998}. 
We further assume that only one agent can be attacked at each time step.\footnote{Although relaxing this assumption is straightforward, we opt to keep it for clarity.} 
%
%For \textbf{multi-agent Q($\bm{\ko}$)}, we can then define the set of control policies as
%We present multi-agent robust learning in this section. We consider the joint action learning (JAL) approach~\citep{claus1998}, where the agents can communicate/observe each other's actions, rewards and policies.
%We assume an intentional and intelligent attack, which minimises the possible return (minimising the Q-function) as explained above. An important aspect is the sparsity of such attacks with the aim not to be too cautious when not needing to be.
%The risk of attack is assumed to be known to the agents and is expressed by the probability of attack per state, which is defined by the control transition function $\mathcal{C}$.
%In our algorithm we use a parameter $\kp$ which expresses the risk of attack and allows the algorithm to learn an accordingly safe strategy $\pi$. Thus, we consider several control domains $c_\sigma \in C$, where
% $C = \{\max_{A_1} \max_{A_2} Q, \min_{A_1} \max_{A_2} Q, \min_{A_2} \max_{A_1} Q\}$
% with corresponding probability vector $p = \{1 - \kp, \nicefrac{\kp}{2}, \nicefrac{\kp}{2}\}$, representing the situation where attacks happen with probability $\kp$ as before, thus each agent can be attacked independently with probability $\nicefrac{\ka}{2}$. 
For \textbf{multi-agent Q($\bm{\ko}$)} we can write Eq.~\eqref{eq:value_fn} for each individual agent as
% \begin{equation*}
% \begin{aligned}
%     V^\ko(s) &= (1-\kp)\max_{A_1} \max_{A_2} Q(s,\langle a_1,a_2\rangle) \\
%     & + \quad \frac{\kp}{2} \quad \, \min_{A_1} \max_{A_2} Q(s, \langle a_1, a_2 \rangle)\\
%     & + \quad \frac{\kp}{2} \quad \, \min_{A_2} \max_{A_1} Q(s, \langle a_1, a_2 \rangle)
% \end{aligned}
% \end{equation*}
\begin{equation*}
\begin{aligned}
    V^\ko(s) = (1-\kp)\max_{A_1} \max_{A_2} Q(s,\langle a_1,a_2\rangle) & + \frac{\kp}{2} \, \min_{A_1} \max_{A_2} Q(s, \langle a_1, a_2 \rangle)\\
    & + \frac{\kp}{2} \, \min_{A_2} \max_{A_1} Q(s, \langle a_1, a_2 \rangle)
\end{aligned}
\end{equation*}
with $a_1 \in A_1$ and $a_2 \in A_2$, representing the scenario in which no attack happens with probability $(1-\kp)$, and each agent is attacked individually with probability $\nicefrac{\kp}{2}$.\footnote{Note the order of the $\min\max$, which follows the Stackelberg assumption of an all-knowing attacker who moves last.} Analogously, we can define Eq.~\eqref{eq:value_fn} for \textbf{multi-agent Expected SARSA($\bm{\ko}$)} as
\begin{equation*}
\begin{aligned}
   V^\ko(s) &= (1- \kp) \qquad \mathbb{E}_{a_1 \sim \pi_1, a_2 \sim \pi_2} \big[ Q(s,\langle a_1,a_2 \rangle)\big]\\  
    & + \quad \frac{\kp}{2} \quad \, \min_{A_1} \mathbb{E}_{a_2 \sim \pi_2} \big[ Q(s, \langle a_1, a_2 \rangle) \big]\\
    & + \quad \frac{\kp}{2} \quad \, \min_{A_2} \mathbb{E}_{a_1 \sim \pi_1} \big[ Q(s, \langle a_1, a_2 \rangle) \big]
\end{aligned}
\end{equation*}
where we now compute an expectation over the actual policy of the agents that are not attacked, while the attacker is still minimizing.

\section{Theoretical Analysis}

% \note[Daan]{I still think we should highlight the second contribution more clearly (convergence for fixed kappa) since the first one (convergence when $\kp \rightarrow 0$) seems somewhat silly at first sight. As in, our method converges when we stop using our method. Or, alternatively, we should make clear why this IS useful.}
% \note[Daan]{We should be really precise here, I think I know what you mean by your optimal / maximal distinction but a reviewer might not. We are talking about optimal in the sense of \textit{optimal w.r.t. the operator} and \textit{maximal w.r.t. the expected return}?}

% \note[Daan]{I changed the heading of the section, happy to change it back though. Just a suggestion.}

In this section we analyze theoretical properties of the proposed $\ko$-methods. We start by relating the different algorithms to each other in the limit of their respective parameters. Then we proceed to show convergence of both Q($\ko$) and Expected SARSA($\ko$) to 
%We base our reasoning on the non-expansion property of the operator, leading to convergence to a fixed point.
%We show convergence to two different fixed points: (i) to the optimal-maximal $Q^\star$ for decreasing parameter $\kp$ and (ii) to optimal-safe fixed point $Q^\star_{\ko}$ for constant parameter $\kp$
two different fixed points: (i) to the optimal value function $Q^\star$ of the original MDP in the limit where $\kp \rightarrow 0$; and (ii) to the optimal robust value function $Q^\star_{\ko}$ of the MDP that is \textit{generalized} w.r.t. the operator $\ko$ for constant parameter $\kp$. 
Note that \textit{optimality} in this sense is purely induced by the relevant operator. In (i) this is the standard Bellman optimality which maximizes the expected discounted return of the MDP. However, in (ii) we derive optimality in the context of \textit{Generalized MDPs}~\cite{szepesvari1997}, where optimal simply means the fixed point of a given operator, which can take many forms. %In our $\ko$ methods the operator has the form of a linear combination of max and min operator.

Before proceeding with the convergence proofs, Figure~\ref{fig:update_target} summarizes some relationships between the algorithms in terms of their targets, in the limit of their respective parameters: As is known, Expected SARSA, SARSA, and Q-learning become identical in the limit of a greedy policy \citep{sutton2018reinforcement,vanseijen2009}. Furthermore, the update targets of our $\ko$-methods approach the update targets of the standard TD methods on which they are based as $\kp \rightarrow 0$. Finally, Expected SARSA($\ko$) and Q($\ko$) share the same relationship as their original versions, and thus Expected SARSA($\ko$) approaches Q($\ko$) as $\epsilon \rightarrow 0$. % These relationships are summarized in Figure~\ref{fig:update_target}.
% \note[Daan]{Think about and maybe reformulate the next sentences.}
Note that the algorithms' equivalence in the limit does not hold in the transient phase of the learning process, and hence in practice they may converge on different paths and to different policies that share the same value function.
%Both of these proofs are very relevant for explaining the meaning of the $\ko$ operator. 
%For the decreasing parameter $\kp$ we can find a parallel in decaying exploration in SARSA, for which SARSA is proven to converge~\cite{singh2000convergence}. The update target of SARSA for $\epsilon \rightarrow 0$ becomes same to the update target of Q-learning and similarly the update target of Q($\ko$) for $\kp \rightarrow 0$ becomes same to the update target of Q-learning as shown in Figure~\ref{fig:update_target}, however the algorithms can converge to different solutions induced be the learning process using decaying parameter $\epsilon$ or $\kp$, and thus 
For a comprehensive understanding of the algorithms introduced in Section~\ref{sec:td-algorithms}, the following sections provide proofs for both convergence of $\ko$ methods for $\kp \rightarrow 0$, as well as their convergence when $\kp$ stays constant.\footnote{While we focus on the adversarial targets considered in Section~\ref{sec:td-algorithms}, a previous proof of convergence under persistent exploration~\citep{szepesvari1997} can be interpreted as a model of random failures with fixed kappa.}
\begin{figure}
\centering
\resizebox {0.8\linewidth} {!} {
\begin{tikzpicture}[node distance=1cm, auto]
 %nodes
 \node[] (dummyStart) {};
 \node[punkt, fill=orange, inner sep=2pt, right=1cm of dummyStart] (Qkappa) {Q($\ko$)};
 \node[punkt, fill=green, inner sep=2pt,left=1cm of dummyStart]
 (expSARSAkappa) {Expected SARSA($\ko$)}
 edge[pil] node[auto] {$\epsilon \rightarrow 0$} (Qkappa.west);
%  edge[pil,bend left=45] node[auto] {$\epsilon \rightarrow 0$} (Qkappa.north);
 \node[punkt, fill=green, inner sep=2pt,below=0.8cm of expSARSAkappa]
 (expSARSA) {Expected SARSA}
  edge[pil,<-]node[auto] {$\kp \rightarrow 0$} (expSARSAkappa.south);
 \node[punkt,fill=orange, inner sep=2pt,below=0.8cm of Qkappa]
 (Q) {Q-learning}
 edge[pil,<-]node[auto] {$\kp \rightarrow 0$} (Qkappa.south)
 edge[pil, <-] node[above] {$\epsilon \rightarrow 0$} (expSARSA.east);
 \node[left=0.5cm of Q](dummy){};
 \node[punkt, fill=green, inner sep=2pt,below=2cm of dummyStart]
 (SARSA) {SARSA}
  edge[pil,<-]node[auto] {$\epsilon \rightarrow 0$} (expSARSA.south)
  edge[pil,->]node[below, xshift=10pt] {$\;\epsilon \rightarrow 0$} (Q.south);
%   \draw[->,ultra thick] (-2,-5) node[left]{on-policy}--(2,-5) node[right]{off-policy};
\end{tikzpicture}
}
\vsp{-1.2em}
\caption{The relationship between the learning targets of different algorithms in the limits of their parameters. On-policy methods are in green, off-policy methods in orange.}
% The target is either induced by the policy (on-policies), or by the operator $\ko$ ($\ko$ methods), or by their combination (Expected SARSA($\ko$)), or just maximizing (Q-learning).} % , using $\epsilon$-greedy policy
\label{fig:update_target}
\vsp{-1em}
\end{figure}
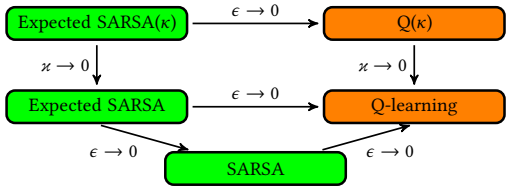
%
%
%Proving convergence of the $\ko$ methods to the optimal-safe fixed point for a fixed parameter $\kp$ is also useful for scenarios where the threat of an attack is stationary, which is the main idea behind designing the $\ko$ methods.
%
% ``The use of the word optimal is somewhat strange since $V^\star$ need not be the largest or smallest value function in any sense; it is simply the fixed point of the dynamic-programming operator T . This terminology comes from the Markov decision process model, where $V^\star$ is the largest value function of all policies and is retained for consistency.''~\cite{szepesvari1997}.
%
% \note[Daan]{We could make a note that a model of random failures with fixed kappa is proven already by \citep{szepesvari1997} (persistent exploration).}

\subsection{Convergence to the Optimal $Q^\star$}
\label{sec:convergence_optimal}

There exist several proofs of convergence for the temporal difference algorithms Q-learning~\cite{jaakkola1994, tsitsiklis1994}, SARSA~\cite{singh2000convergence}, and Expected SARSA~\cite{vanseijen2009}. Each of these proofs hinges on linking the studied algorithm to a stochastic process, and then using convergence results from stochastic approximation theory \cite{dvoretzky1956,robbins1951}. These proofs are based on the following lemma, presented as Theorem 1 in~\citet{jaakkola1994} and as Lemma 1 in~\citet{singh2000convergence}. These differ in the third condition, which describes the contraction mapping of the operator. The contraction property used for the Q-learning proof~\cite{jaakkola1994} has the form $||\mathbb{E}\{F_t(\cdot)|P_t\}|| \leq \gamma ||\Delta_t||$, where $\gamma \in [0,1)$. We show the lemma as it was used for the SARSA proof provided by~\citet{singh2000convergence}, who show that the contraction property does not need to be strict; strict contraction is required to hold only asymptotically.%\footnote{For proof based on Proposition 4.5 of~\citet{bertsekas1995} see Appendix A of~\citet{singh2000convergence}.}:
% : \note[Daan]{Maybe we should make more clear why we can use the non-strict version.}
% \note[Richard]{unify $s'$ and $s_{t+1}$}

\begin{lemma}
Consider a stochastic process $(\alpha_t, \Delta_t, F_t), t \geq 0$, where $\alpha_t, \Delta_t, F_t: X \rightarrow \mathbb{R}$ satisfy the equations
\begin{equation*}
\Delta_{t+1}(x) = \big(1 - \alpha_t(x)\big)\Delta_t(x) + \alpha_t(x)F_t(x), \, x \in X, \, t = 0,1,2,\dots   
\end{equation*}
Let $P_t$ be a sequence of increasing $\sigma$-fields such that $\alpha_0$ and $\Delta_0$ are $P_0$-measurable and $\alpha_t, \Delta_t$ and $F_{t-1}$ are $P_t$-measurable, $t=1,2,\dots$. Then, $\Delta_t$ converges to zero with probability one (w.p.1) under the following assumptions:
\begin{enumerate}
    \item the set $X$ is finite,
    \item $0 \leq \alpha_t(x_t) \leq 1$, $\sum_t \alpha_t(x_t) = \infty$, $\sum_t \alpha^2_t(x_t) < \infty$ w.p.1,
    % \addtocounter{enumi}{1}
    % \myitemA  $||\mathbb{E}\{F_t(\cdot)|P_t\}|| \leq \gamma ||\Delta_t||$, where $\gamma \in [0,1)$,
    % \myitemB  $||\mathbb{E}\{F_t(\cdot)|P_t\}|| \leq \gamma ||\Delta_t|| + c_t$, where $\gamma \in [0,1)$ and $c_t$ converges to zero w.p.1,
    \item $||\mathbb{E}\{F_t(\cdot)|P_t\}|| \leq \gamma ||\Delta_t|| + c_t$, where $\gamma \in [0,1)$ and $c_t$ converges to zero w.p.1,
    \item $Var\{F_t(x_t)|P_t\} \leq K(1 + ||\Delta_t||)^2$, where $K$ is some constant,
\end{enumerate}
where $||\cdot||$ denotes a maximum norm.
\label{lemma_stoch}
\end{lemma}

The proof continues by relating Lemma~\ref{lemma_stoch} to the temporal difference algorithm, following the same reasoning as \citet{vanseijen2009} in their convergence proof for Expected SARSA. We define $X = S \times A$, $P_t = \{Q_0, s_0, a_0, r_0, \alpha_0, s_1, a_1, \dots, s_t, a_t\}, x_t = (s_t, a_t)$, which represents the past at step $t$ and $\alpha_t(x_t) = \alpha_t(s_t, a_t)$ is a learning rate for state $s_t$ and action $a_t$. To show the convergence of $Q$ to the optimal fixed point $Q^\star$ we set $\Delta_t(x_t) = Q_t(s_t, a_t) - Q^\star(s_t, a_t)$, therefore when $\Delta_t$ converges to zero, then the $Q$ values converge to $Q^\star$. The maximum norm $||\cdot||$ can be expressed as maximizing over states and actions as $||\Delta_t||=\max_s \max_a |Q_t(s,a) - Q^\star(s, a)|$.

% \note[Daan]{If I'm correct we can leave out this next part. Our proofs work without these facts.}
% \note[Remove?]{We make use of the fact that our Q($\ko$) method approaches standard Q-learning for decreasing parameter $\kp$, where Q($\ko$) becomes Q-learning for $\kp=0$. Analogically, the same holds for Expected SARSA($\ko$) and Expected SARSA.}

We follow the reasoning of Theorem 1 from~\citet{vanseijen2009}, where we repeat the conditions (1), (2) and (4) and modify the condition \textbf{(3)} for the $\ko$ methods as:
\begin{theorem}
% \note[Daan]{Here we could make more specific what we prove, e.g. "}
Q($\ko$) and Expected SARSA($\ko$) as defined in Section \ref{sec:single-agent-kappa} using the respective value function $V^\ko$, defined by 
% The temporal difference learning with operator $\ko$ given by 
\begin{equation*}
    Q_{t+1}(s_t, a_t) = (1 - \alpha_t(s_t, a_t))Q_t(s_t, a_t) + \alpha_t(s_t, a_t)[r_t + \gamma V^\ko_t(s_{t+1})]
\end{equation*}
converge to the optimal Q function $Q^\star(s, a)$ if:
\begin{enumerate}
    \item the state space $S$ and action space $A$ are finite,
    \item $\alpha_t(s_t, a_t) \in (0, 1)$, $\sum_t \alpha_t(s_t, a_t) = \infty$ and $\sum_t \alpha^2_t(s_t, a_t) < \infty$ w.p.1,
    % \item
    \itembf $\kp$ converges to zero w.p.1,
    \myitemA for Expected SARSA($\ko$) the policy is greedy in the limit with infinite exploration (GLIE assumption),
    \item the reward function is bounded.
\end{enumerate}
\label{theorem_convergence}
\end{theorem}

\begin{proof}{Convergence of Q($\ko$):}
To prove convergence of Q($\ko$) we have to show that the conditions from Lemma~\ref{lemma_stoch} hold. Conditions (1), (2) and (4) of Theorem~\ref{theorem_convergence} correspond to conditions (1), (2) and (4) of Lemma~\ref{lemma_stoch} \cite{vanseijen2009}. We now need to show that the contraction property holds as well, using condition (3) of Theorem~\ref{theorem_convergence}. Adapting the proof of~\citet{vanseijen2009}, we set $F_t(x) = F_t(s, a) = r_t(s, a) + \gamma V^\ko_t(s') - Q^\star(s, a)$ to show that $F_t(s, a)$ is a contraction mapping, i.e., condition (3) in Lemma~\ref{lemma_stoch}.
For Q($\ko$) we write:
\begin{equation*}
    F_t = r_t + \gamma \big( (1-\kp) \max_{a} Q_t(s_{t+1},a) + \kp \min_{a} Q_t(s_{t+1}, a) \big) - Q^\star(s_t, a_t).
\end{equation*}
We want to show that $||\mathbb{E}\{F_t\}|| \leq \gamma ||\Delta_t|| + c_t$ to prove the convergence of Q($\ko$) to the optimal value $Q^\star$. 
% \note[Karl]{This requires more explanation}This reasoning is based on Expected SARSA proof in~\citet{vanseijen2009} and follows as:
% \begin{equation*}
% \resizebox{1\hsize}{!}{$
% \begin{split}
%         & ||\mathbb{E}\{F_t\}|| \\
%         & = ||\mathbb{E}\{r_t + \gamma \big( (1-\kp) \max_{a} Q(s_{t+1},a) + \kp \min_{a} Q(s_{t+1}, a) \big) - Q^\star(s_t, a_t)\} ||\\
%         & \leq  ||\mathbb{E}\{r_t + \gamma \max_a Q_t(s_{t+1}, a) - Q^\star(s_t, a_t)\}|| +\\
%         & \quad \gamma ||\mathbb{E}\{(1 - \kp)\max_a Q_t(s_{t+1}, a) + \kp \min_a Q_t(s_{t+1}, a) - \max_a Q_t(s_{t+1}, a)\}||\\
%         & \leq \gamma \max_s |\max_a Q_t(s, a) - \max_a Q^\star(s, a)| +\\
%         & \quad \gamma \max_s |(1 - \kp)\max_a Q_t(s, a) + \kp \min_a Q_t(s, a) - \max_a Q_t(s, a)|\\
%         & \leq \gamma ||\Delta_t|| +\\
%         & \quad \gamma \kp \max_s |\max_a Q_t(s,a) - \min_a Q_t(s,a)|,
% \end{split}$}
% \end{equation*}
% \note[Daan]{I think you can omit the introduction of the extra $\max_a Q_t(s,a)$ term, and immediately split it out of the original $(1-\kp)\max_a Q_t(s,a)$. Also, in the last line I think it should be $|\min - \max|$ rather than $|\max - \min|$?}
\begin{equation*}
\resizebox{1\hsize}{!}{$
\begin{split}
        & ||\mathbb{E}\{F_t\}|| \\
        & = ||\mathbb{E}\{r_t + \gamma \big( (1-\kp) \max_{a} Q_t(s_{t+1},a) + \kp \min_{a} Q_t(s_{t+1}, a) \big) - Q^\star(s_t, a_t)\} ||\\
        & \leq  ||\mathbb{E}\{r_t + \gamma \max_a Q_t(s_{t+1}, a) - Q^\star(s_t, a_t)\}|| +\\
        & \quad \gamma ||\mathbb{E}\{\kp \min_a Q_t(s_{t+1}, a) - \kp\max_a Q_t(s_{t+1}, a)\}||\\
        & \leq \gamma \max_s |\max_a Q_t(s, a) - \max_a Q^\star(s, a)| +\\
        & \quad \gamma \max_s |\kp \min_a Q_t(s, a) - \kp\max_a Q_t(s, a)|\\
        & \leq \gamma ||\Delta_t|| + \gamma \kp \max_s |\min_a Q_t(s, a) - \max_a Q_t(s, a)|,
\end{split}$}
\end{equation*}
% \note[Daan]{I still think it would be more intuitive to not swap the min and max in the last line ;) (and if so, don't forget to change the text below as well).}
%
% \note[Daan]{I also think I know where the $r_t$ goes. If you look in the Van Seijen paper, they define the Expected SARSA update in the usual sense on the tuple $\langle s_t, a_t, r_{t+1}, s_{t+1} \rangle$ (Equation 5). However when defining $F_t$ they base this on the sequence $P_t=\{Q_0,s_0,a_0,r_0,\alpha_0,s_1,a_1,\ldots\}$ i.e. they give the first reward index 0 instead of 1. So, the $r_t$ in the second line above is $r_{t+1}$ in the usual sense... and IF that is indeed the case, then we can substitute $r_t=Q^\star(s_t,a_t)-\gamma \max_a Q^\star(s_{t+1},a)$ (by the definition of $Q^\star$). The two remaining $Q^\star(s_t,a_t)$ terms cancel each other out, and we're left with line 5 (taking into account that "the maximal difference in values over all states is at least as large as a difference between values given in state $s_{t+1}$".}
%
where the first inequality follows from standard algebra and the fact that splitting the maximum norm yields at least as large a number, %and adding and subtracting a term ($\max_a Q_t(s_{t+1},a)$) does not change the expression,
the second inequality follows from the definition of $Q^\star$ and the maximal difference in values over all states being at least as large as a difference between values given in state $s_{t+1}$, and the third inequality follows from the definition of $||\Delta_t||$ above.\footnote{Recall that we set out in this section to show convergence to the same optimal Q-value as classical Q-learning $Q^\star(s_t, a) = r_t + \gamma \max_{a'} Q^\star(s_{t+1}, a')$, even if we do so by our new operator.} 
We can see that if we set $c_t =  \gamma \kp \max_s |\min_a Q_t(s,a) - \max_a Q_t(s,a)|$, then for $\kp \rightarrow 0$ we get $c_t$ converging to zero w.p.1, thus proving convergence of Q($\ko$). 
\end{proof}

\begin{proof}{Convergence of Expected SARSA($\ko$):}
% \paragraph{On-policy learning - Expected SARSA($\ko$)}
Similarly as in the proof of Q($\ko$) %the conditions 1), 2) and 4) of Theorem~\ref{theorem_convergence} correspond to conditions 1), 2) and 4) of Lemma~\ref{lemma_stoch}. We 
we need to show that the contraction property holds as well, this time using conditions (3) and (3a) of Theorem~\ref{theorem_convergence}.
% \note[Not needed?]{We make use of the fact that Expected SARSA($\ko$) for parameter $\kp = 0$ becomes the standard Expected SARSA for which \citet{vanseijen2009} proved convergence to optimal $Q^\star$, using the proof of SARSA from~\citet{singh2000convergence}. Therefore, we can show that Expected SARSA($\ko$) also converges to optimal value $Q^\star$ for parameter $\kp \rightarrow 0$. }
We first define:
\begin{equation*}
\resizebox{1\hsize}{!}{$
\begin{split}
F_t = r_t + \gamma \big( (1- \kp) \sum_a \pi_t(a|s_{t+1}) Q_t(s_{t+1}, a) + \kp \min_{a} Q_t(s_{t+1}, a) \big) - Q^\star(s_t, a_t)
\end{split}$}
\end{equation*}
and then show the following:
\begin{equation*}
\resizebox{1\hsize}{!}{$
\begin{split}
        & ||\mathbb{E}\{F_t\}||\\
        & = ||\mathbb{E}\{r_t + \gamma \big( (1- \kp) \sum_a \pi_t(a|s_{t+1}) Q_t(s_{t+1}, a) + \kp  \min_{a} Q_t(s_{t+1}, a) \big) - Q^\star(s_t, a_t)\}||\\
        & \leq  ||\mathbb{E}\{r_t + \gamma \max_a Q_t(s_{t+1}, a) - Q^\star(s_t, a_t)\}|| +\\
        & \quad \gamma ||\mathbb{E}\{(1- \kp) \sum_a \pi_t(a|s_{t+1}) Q_t(s_{t+1}, a) + \kp  \min_{a} Q_t(s_{t+1}, a) - \max_a Q_t(s_{t+1}, a)\}||\\
        & \leq \gamma \max_s |\max_a Q_t(s, a) - \max_a Q^\star(s, a)| +\\
        & \quad \gamma \max_s |(1- \kp) \sum_a \pi_t(a|s) Q_t(s, a) + \kp  \min_{a} Q_t(s, a) - \max_a Q_t(s, a)|,\\
\end{split}$}
\end{equation*}
where the inequalities use the same operations as above in the proof of Q($\ko$).
If we set $c_t =  \gamma \max_s |(1- \kp) \sum_a \pi_t(a|s) Q_t(s, a) + \kp  \min_{a} Q_t(s, a) - \max_a Q_t(s, a)|$ and assume that the policy is greedy in the limit with infinite exploration (GLIE assumption) and parameter $\kp \rightarrow 0$ w.p.1 (conditions (3) and (3a)), it follows that $c_t$ converges to zero w.p.1, thereby proving that Expected SARSA($\ko$) converges to optimal fixed point $Q^\star$.
\end{proof}

% We now show convergence of $\ko$ methods for fixed parameter $\kp$.

% \note[Richard]{Unify the notation, e.g., transition function, operator $T$, next state, etc.}
% \note[Richard]{How about changing $Q^\star_{safe}$ to $Q^\star_{\ko}$?}
% \note[Daan]{Good idea!}

\subsection{Convergence to the Robust $Q^\star_{\ko}$}\label{sec:convergence_safe}

In this section we show convergence to %a fixed safe point which is optimal with respect to the operator $\ko$ representing the probability of attack per state.
the robust value function $Q^\star_{\ko}$ which is optimal w.r.t. the operator $\ko$. The main difference with the proof of Theorem~\ref{theorem_convergence} is that here we do not require $\kp \rightarrow 0$ but instead assume it remains constant over time. 
We base our reasoning on the theory of \textit{Generalized MDPs} \citep{szepesvari1997}.
%\citet{szepesvari1997} presented \textit{Generalized MDPs}, where they derive convergence proofs for a wide range of value operators. The necessary condition is that the operator is non-extraction.
A Generalized MDP is defined using operator-based notation as 
\begin{equation*}
    \Big( \bigotimes \bigoplus (R + \gamma V) \Big)(s) = \max_a \sum_{s'} P(s'|s, a)\big( R(s, a) + \gamma V(s')\big),
\end{equation*}
where the operator $\bigotimes$ defines how an optimal agent chooses her actions (in the classic Bellman equation this denotes maximization) and operator $\bigoplus$ defines how the value of the current state is updated by the value of the next state (in the classic Bellman equation this denotes a probability weighted average over the transition function). These operators can be chosen to model various different scenarios. 
The generalized Bellman equation can now be written as $V^\star = \bigotimes \bigoplus(R + \gamma V^\star)$.
The main result of \citet{szepesvari1997} is that if $\bigotimes$ and $\bigoplus$ are non-expansions, then there is a unique optimal solution to which the generalized Bellman equation converges, given certain assumptions. For $0 \leq \gamma < 1$ and non-expansion properties of $\bigotimes$ and $\bigoplus$ we get a contraction mapping of the Bellman operator $\mathcal{T}$ defined as $\mathcal{T}V = \bigotimes \bigoplus (R + \gamma V)$. Then, the operator $\mathcal{T}$ has a unique fixed point by the Banach fixed-point theorem~\cite{smart1974}.

%\note[Daan]{The next part needs to be rewritten a bit more clearly.}

%\note[Richard]{I am a bit confused by the notation of the next state, $s'$ or $s_{t+1}$, in the Generalized MDPs they use $y$ for the next state and then they use $y_t$, I guess that's a value of the next state in time $t$, we should make this somehow consistent throughout the proofs. Same for $Q'$ or $Q_{t+1}$.}
%\note[Daan]{To my understanding the s' (and Q') notation is used when the time index is irrelevant, meaning you talk about *any* possible difference between a current and next state / value. When dealing with a specific update at time $t$, the $s_{t+1}$ is used.}

Building on the stochastic approximation theory results (as we also used in the Section~\ref{sec:convergence_optimal}), \citet{szepesvari1997} show the following:
\begin{lemma}
Generalized Q-learning with operator $\bigotimes$ using Bellman operator $\mathcal{T}_t(Q', Q)(s, a) =$
\begin{equation*}
\resizebox{1\hsize}{!}{$
%\begin{aligned}
    %&\mathcal{T}_t(Q', Q)(s, a) =\\
    %&
    \begin{cases}
    \big(1 - \alpha_t(s, a)\big) Q'(s, a) + \alpha_t(s, a)\big(r_t + \gamma(\bigotimes Q)(s'_t)\big) & \text{if~} s = s_t, a = a_t\\
    Q'(s,a) & \text{otherwise}
    \end{cases}
%\end{aligned}
$}
\end{equation*}
converges to the optimal Q function w.p.1, if
\begin{enumerate}
    \item $s'_t$ is randomly selected according to the probability distribution defined by $P(s_t, a_t, \cdot)$,
    % \item the learning rates are decayed so that $\sum_t \alpha_t(s, a) = \infty$ and $\sum_t \alpha^2_t(s, a) < \infty$, w.p.1.
    \item $\alpha_t(s_t, a_t) \in (0, 1)$, $\sum_t \alpha_t(s_t, a_t) = \infty$ and $\sum_t \alpha^2_t(s_t, a_t) < \infty$ w.p.1,
    \item $\bigotimes$ is a non-expansion,
    \item the reward function is bounded.
\end{enumerate}
\label{lemma:generalized-q}
\end{lemma}
%As discussed above in Section~\ref{sec:convergence_optimal}, this is based on the common temporal difference learning convergence proofs using the convergence of stochastic approximation algorithms. 
We base our convergence proofs for Q($\ko$) and Expected SARSA($\ko$) on the insights of \citet{szepesvari1997} given in Lemma~\ref{lemma:generalized-q}.

% \note[Daan]{I added this to make things more formal.}
\begin{theorem}
Q($\ko$) and Expected SARSA($\ko$) as defined in Section \ref{sec:single-agent-kappa} converge to the robust Q function $Q^\star_\ko$ for any fixed $\kp$.
\label{theorem_convergence_robust}
\end{theorem}

\begin{proof}{Convergence of Q($\ko$) to $Q^\star_{\ko}$:}
To prove convergence of Q($\ko$) we follow the proof of Generalized Q-learning in Lemma~\ref{lemma:generalized-q}.
The only condition we need to guarantee is the non-expansion property of the operator in the value function update, which for Q($\ko$) is a weighted average of the operators \textit{min} and \textit{max}. 
We write the operator $\bigotimes$ for Q($\ko$) as $\bigotimes^\ko$ and define it as
\[
(\textstyle{\bigotimes^\ko} Q)(s, a) = (1-\kp) \max_a Q(s,a) + \kp \min_a Q(s, a).
\]
%= & (1- \kp) \mathbb{E}_{a^\star \sim \pi} \big[ Q^\pi(s, a^\star ) \big] + \kp  \min_{A} Q^\pi(s, a)
In Appendix B of~\citet{szepesvari1997}, Theorem 9 states that any linear combination of non-expansion operators is also a non-expansion operator. %is that if the operators are non-expansions componentwise then they are non-expansions combined as well, 
Moreover Theorem 8 states that the summary operators $\max$ and $\min$ are also non-expansions. %''\footnote{Note, that some other well-known operators like Boltzmann exploration are not non-expansions, but non-expansion alternatives have been proposed, e.g.,~\cite{asadi2017}.}, 
Therefore, $\bigotimes^\ko$ is a non-expansion as well, thus proving the convergence of Q($\ko$) to the robust fixed point $Q^\star_\ko$ induced by the operator $\ko$.
\end{proof}

% \note[Daan]{We should somewhere discuss the relation between our kappa methods and the "risk-sensitive models" of Section 4.4 in \citep{szepesvari1997}.}

% \note[Richard]{Have a look at the following proof of Expected SARSA($\ko$), I am not sure whether we can combine the two proofs of SARSA and Generalized MDPs in this way? My intuition is that we can, but better check it out. It might need a better argumentation or explanation?}
% \todo[inline]{Following some remarks/questions:}
% \begin{itemize}
%     % \item Asynchronous stochastic process reduced to a synchronous one
%     % \item Optimal policy is always a stationary deterministic policy?
%     \item Have a look at~\citet{szepesvari1997}, page 8 at the top, they discuss allowing stochastic policy but $\min_a f(s,a) \leq (\bigotimes f)(s) \leq \max_a f(s,a)$ where operator $\bigotimes$ defines the stochastic policy. \note[Daan]{Not sure what you want to use this for?}
%     \item Must the operator be stationary? Is a different proof needed for Expected SARSA($\ko$)? \note[Daan]{Not sure what you mean by stationary here.}
% \end{itemize}

\begin{proof}{Convergence of Expected SARSA($\ko$) to $Q^\star_{\ko}$:}
%\note[Daan]{I now think the convergence of Expected SARSA($\ko$) follows by combining the results of \citet{szepesvari1997} in Section 4.5 on persistent exploration with our result for $Q(\ko)$. Then, we don't even need GLIE. We can say Expected SARSA($\ko$) with $\epsilon$-greedy exploration converges to a policy that is optimal w.r.t. the combination of fixed $\kp$ and fixed $\epsilon$, i.e. optimal w.r.t. the attack model AND the epsilon greedy behaviour policy.}
%\note[Richard]{Have a look at the following, whether it is what you meant...}
We base our convergence proof of Expected SARSA($\ko$) again on the work of \citet{szepesvari1997}, this time on their insights regarding persistent exploration (Section 4.5 in their paper). They show that Generalized Q-learning with $\epsilon$-greedy action selection converges, for a fixed $\epsilon$, in the Generalized MDP. Following similar reasoning, we define the operator $\bigotimes$ for Expected SARSA($\ko$) with fixed $\epsilon$ as
\begin{equation*}
\resizebox{1\hsize}{!}{
$ (\bigotimes^\ko Q)(s,a) = (1 - \kp)\left(\epsilon \frac{1}{|A|} \sum_a Q(s, a) + (1 - \epsilon) \max_a Q(s, a)\right) + \kp \min_{a} Q(s, a). $
}
\end{equation*}
Again, from repeated application of Theorems 8 and 9 in Appendix B of~\citet{szepesvari1997} it follows that $\bigotimes^\ko$ is a non-expansion as well. %property which combined with the discount parameter $\gamma \in [0,1)$ quarantees the contraction property from Lemma~\ref{lemma_stoch} (condition 3), 
Therefore, by Lemma~\ref{lemma:generalized-q}, Expected SARSA($\ko$) converges to $Q^\star_{\ko}$ for fixed exploration $\epsilon$.
\end{proof}

% \note[Daan]{If we assume GLIE.. I'm not sure. Intuitively I would think it converges as well, since Expected SARSA($\ko$) slowly turns into Q($\ko$), given that GLIE is satisfied by simply decaying $\epsilon$. But the paper on GMPDs doesn't say anything about operators that change over time.. (I guess that's what you meant by non-stationary above as well).}

% \note[Richard]{Get rid of the following?}

% \note[Daan]{Either we get rid of it, or we use some of it as an idea (not a proof!) for how a *potential proof*  for $\epsilon \rightarrow 0$ could work?}

It remains an open question whether Expected SARSA($\ko$) also converges for decreasing $\epsilon$, e.g., under the GLIE assumption, even though we conjecture that it might.

% \note[DELETE?]{
% To prove the convergence of Expected SARSA($\ko$) to the optimal safe fixed point we combine the proofs of Expected SARSA~\cite{vanseijen2009} (which uses the proof of SARSA~\cite{singh2000convergence}) and the proof of Generalized MDPs~\cite{szepesvari1997}. The necessary condition of the convergence is that the policy is GLIE (greedy in the limit with infinite exploration), which is a common assumption for proofs of on-policy learning methods. We use the relaxed condition on the contraction mapping from~\citet{singh2000convergence}: the contraction property does not need to be strict; strict contraction is required to hold only asymptotically. Therefore, the proof to the optimal safe fixed point follows from condition 3b in Lemma~\ref{lemma_stoch} and the fact that a weighted average of min and max is a non-expansion~\cite{szepesvari1997}, satisfying the contraction property asymptotically for $\epsilon \rightarrow 0$ and thus, proving the convergence of Expected SARSA($\ko$) to a fixed safe point induced by the operator $\ko$.
% }

\subsection{Convergence in the Multi-Agent Case}

We now prove convergence of the cooperative multi-agent variant of the $\ko$ methods presented in Section~\ref{sec:multi-agent-kappa}. This proof builds on the theory of Generalised MDPs, similar to the proofs presented in Section~\ref{sec:convergence_safe}. Therefore this proof also assumes a fixed probability of attack $\kp$. In addition, we make use of the assumption that agents can communicate freely in the learning phase, and thus receive identical information and can build a common joint-action Q-table.
\vsp{-1em}
\begin{theorem}
Multi-agent Q($\ko$) and Expected SARSA($\ko$) as defined in Section \ref{sec:multi-agent-kappa} converge to the robust Q function $Q^\star_\ko$ for any fixed $\kp$.
\label{theorem_convergence_multi}
\end{theorem}
\vsp{-0.8em}
\begin{proof}
The $\bigotimes^\ko$ operator for our multi-agent versions of Q($\ko$) and Expected SARSA($\ko$) consists of a nested combination of different components, in particular $\max_a Q(s,a)$, $\min_a Q(s,a)$, and $\sum_a \pi^\epsilon(s,a)Q(s,a)$ where $\pi^\epsilon$ is the $\epsilon$-greedy policy. By Theorem 8 of \citet{szepesvari1997}, $\max$ and $\min$ are non-expansions. By Theorem 9 of \citep{szepesvari1997}, linear combinations of non-expansion operators are also non-expansion operators. Finally, by Theorem 10 of \citep{szepesvari1997}, products of non-expansion operators are also non-expansion operators. Therefore, also $\max\max$, $\max\min$, and $\min\max$ are non-expansion operators, as are linear combinations of those compounds. Similarly, $\sum_a \pi^\epsilon(s,a)Q(s,a)$ for fixed $\epsilon$ can be written as a linear combination of summary operators, which by Theorems 8 and 9 of \citet{szepesvari1997} is a non-expansion. Therefore, the $\bigotimes^\ko$ operator used in both multi-agent Q($\ko$) and Expected SARSA($\ko$) is a non-expansion. Thus, by Lemma~\ref{lemma:generalized-q}, Q($\ko$) and Expected SARSA($\ko$) converge to $Q^\star_\ko$ for fixed $\kappa$, and in the case of Expected SARSA($\ko$), for fixed $\epsilon$.
\end{proof}

\section{Experiments and Results}\label{sec:experiments}
% \note{stress out the results are statistically significant}
% \note[Richard]{Consider changing the operator $\ko$ in the figures!}
    % \begin{itemize}
    %     \item Classic cliff walking~\cite{sutton2018reinforcement}
    %     \item Cliff walking with stochastic environment
    %     \item Cliff walking with adversarial attack
    %     \item An experiment with multiple agents?? Maybe get back to puddle world and also show the early performance
    %     \item An experiment showing the ``safe exploration''?
    % \end{itemize}

In this section we evaluate temporal difference methods with the proposed operator $\ko$; off-policy type of learning Q($\ko$) and on-policy type of learning Expected SARSA($\ko$). We experiment with a classic cliff walking scenario for the single-agent case and a multi-agent puddle world scenario. Both these domains contain some critical states, a cliff and a puddle respectively, which render very high negative reward for the agent(s) in case of stepping into them. These critical states represent the significant rare events (SREs). We compare our methods to classic temporal difference methods like SARSA, Q-learning and Expected SARSA. In all the experiments we consider an undiscounted ($\gamma$=1), episodic scenario.

\vsp{-0.5em}
\paragraph{Cliff Walking: single-agent}
The Cliff Walking experiment as shown in Figure~\ref{fig:cliff_walking_map} is a classic scenario proposed in~\citet{sutton2018reinforcement} and used frequently ever since (e.g.,~\cite{vanseijen2009}). The agent needs to get from the start state [S] to the goal state [G], while avoiding stepping into the cliff, otherwise rendering a reward of $-100$ and sending him back to the start. For every move which does not lead into the cliff the agent receives a reward of $-1$.
\begin{figure}[tb]
    \centering
    \includegraphics[width=0.5\linewidth, trim={0 0 0 0},clip]{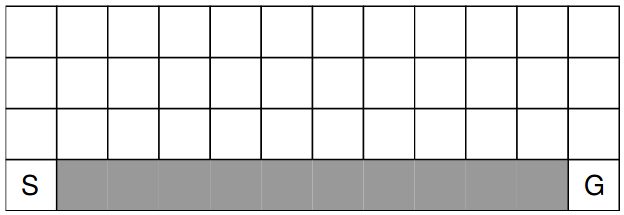}
    \vsp{-1.2em}
    \caption{The Cliff Walking: The agent needs to get from the start [S] to the goal [G], avoiding the cliff (grey tiles).}
    \label{fig:cliff_walking_map}
    \vsp{-1em}
\end{figure}

\vsp{-0.5em}
\paragraph{Puddle World: multi-agent}
The Puddle World environment is a grid world with puddles which need to be avoided by the joint-action learning agents. The two agents jointly control the movement of a single robot in the Puddle World, each controlling either direction $\langle \text{up}, \text{down} \rangle$ or $\langle \text{left}, \text{right} \rangle$. Agent 1 can take the actions \{\textit{stay, move down, move up}\} and agent 2 can choose \{\textit{stay, move left, move right, move right by 2}\}, thus their action spaces are different, further complicating the learning process compared to the single-agent scenario. The joint action is the combination of the two selected actions. We assume a reward of -1 for every move and -100 for stepping into a puddle (returning to the start node). The agents have to move together from the start node at the top left corner to the goal at the bottom right corner. Figure~\ref{fig:puddle_world_map} shows the policy learned by our proposed algorithm Q($\ko$) for the two joint-learning agents. Note how a safer path (longer, avoiding the puddles) is learned with increasing parameter $\kp$ (i.e., higher probability of SREs). For $\kp=0$ our algorithm degenerates to Q-learning (left panel).

\begin{figure}[tb]
\centering
\includegraphics[width=0.3\linewidth, trim={2.5cm 0.4cm 3cm 0.5cm}, clip]{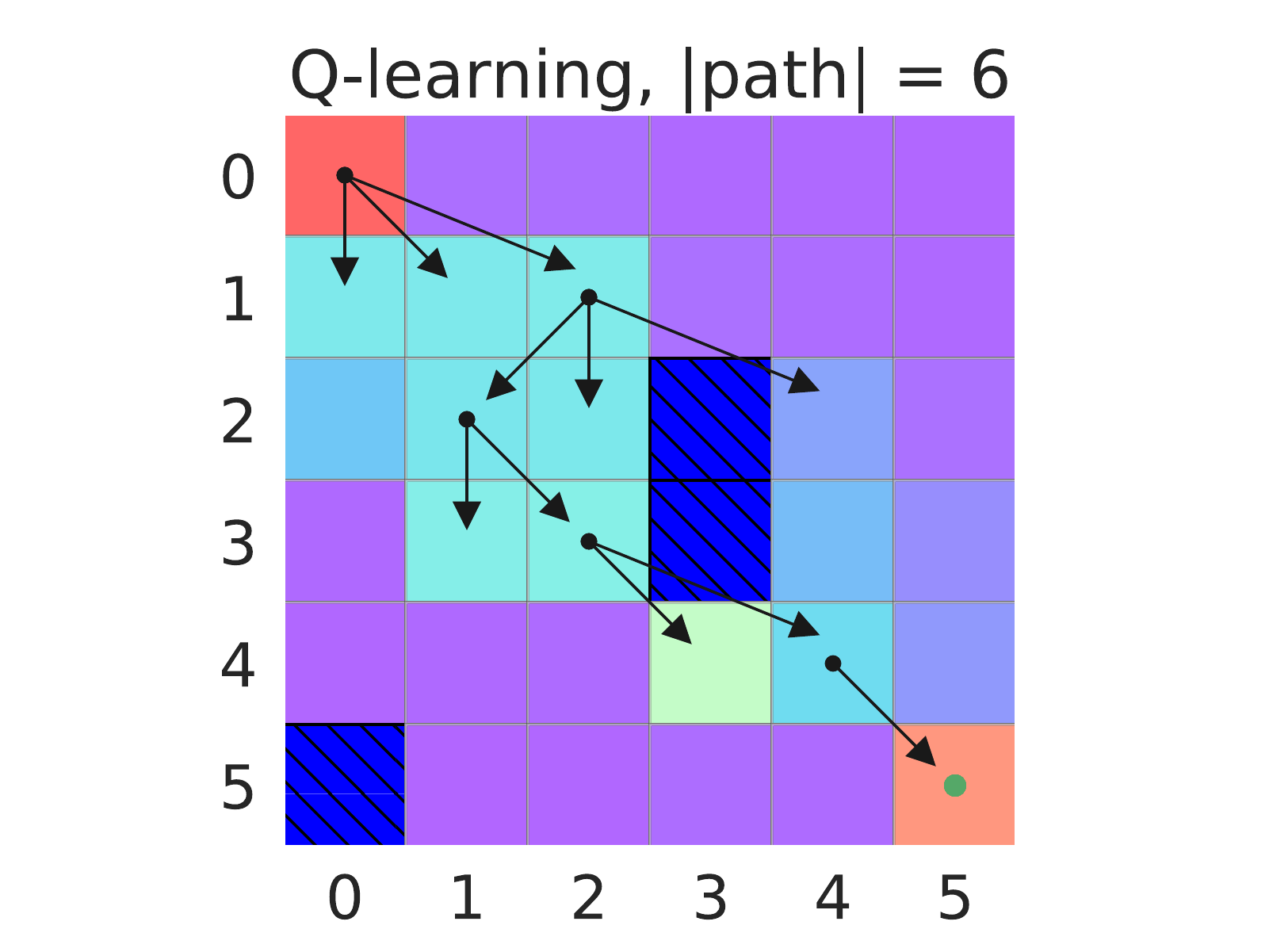}
\includegraphics[width=0.3\linewidth, trim={2.5cm 0.4cm 3cm 0.5cm}, clip]{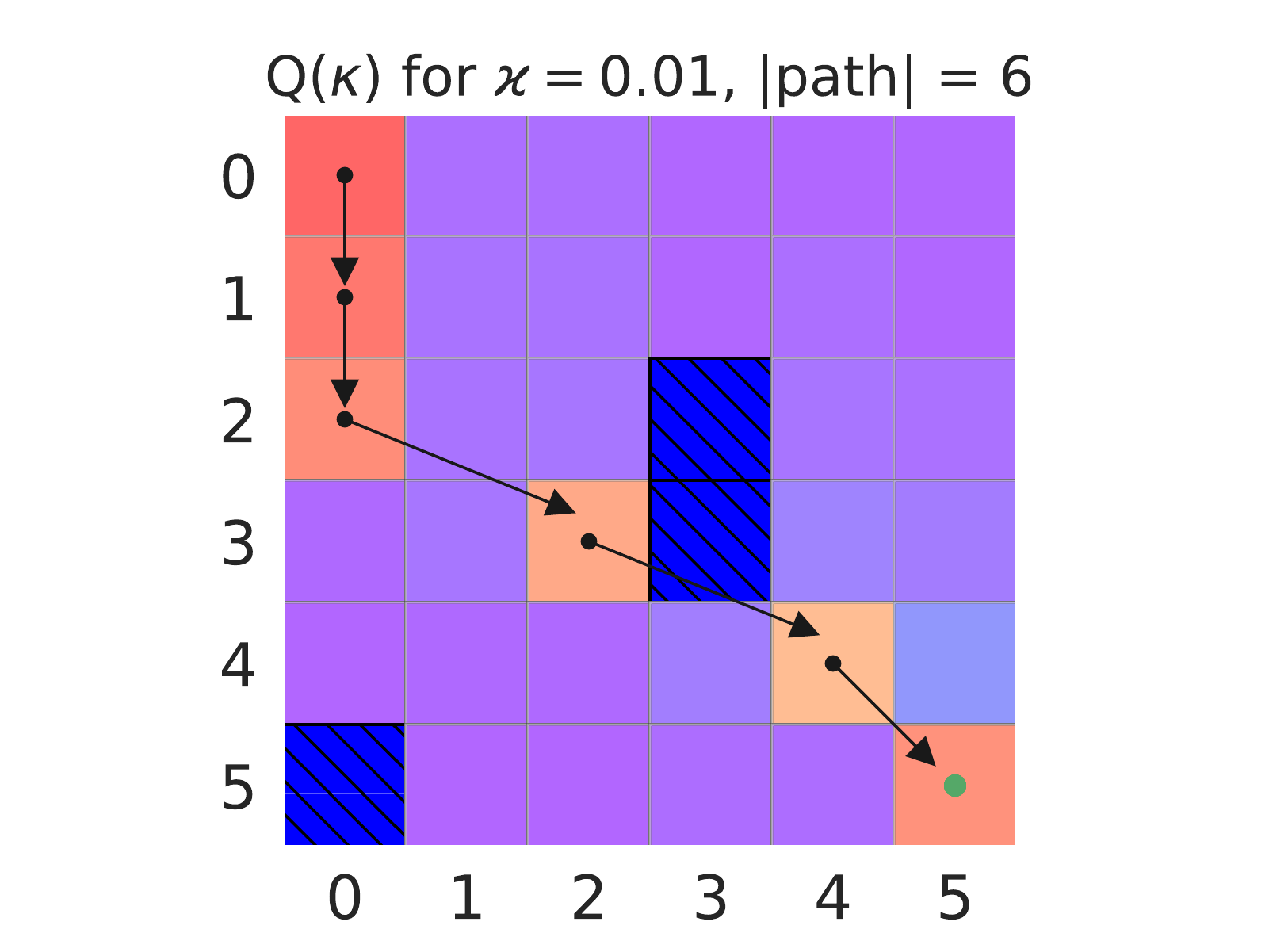}
\includegraphics[width=0.3\linewidth, trim={2.5cm 0.4cm 3cm 0.5cm}, clip]{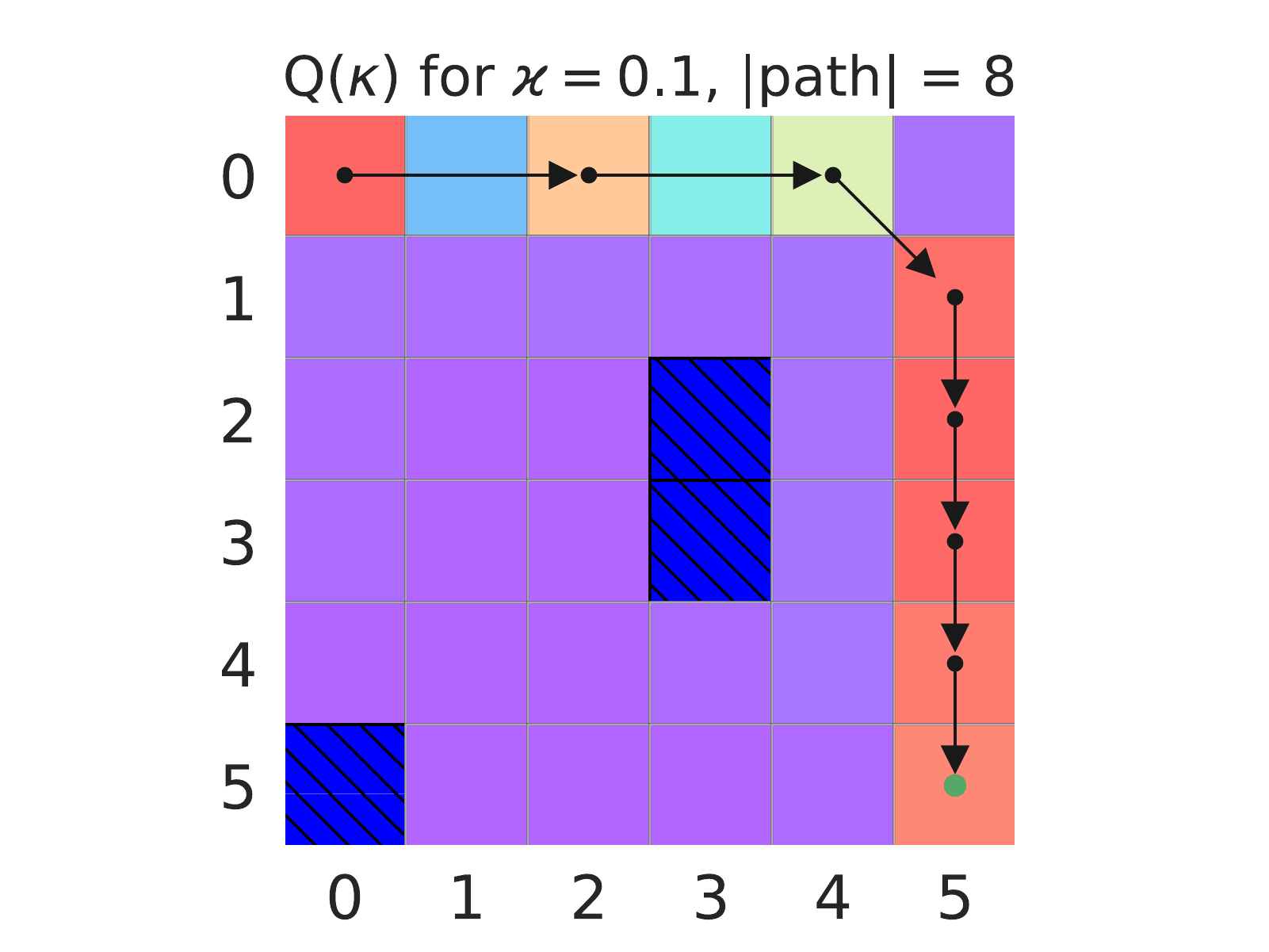}
\vsp{-1.2em}
% \caption{Puddle world}
\caption{The Puddle World: Q($\ko$) learns a safer path with increasing $\kp$. Puddles are dark blue, the arrows show the optimal actions on the learned path, and the heatmap shows the number of visits to each state (\protect\includegraphics[height=0.5em]{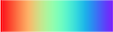}, blue is none).}
\label{fig:puddle_world_map}
\vsp{-1.4em}
\end{figure}

\subsection{Performance}
We replicate the experiment of~\citet{vanseijen2009} on the Cliff Walking domain, in which we compare our $\ko$ methods with Q-learning, SARSA and Expected SARSA, and perform a similar experiment on the Puddle World domain. In line with \citep{vanseijen2009} we show (i) early performance, which is the average return over the first $100$ training episodes, and (ii) converged performance, which is the average return over $100,000$ episodes.
% \begin{figure*}[!b]
\begin{figure*}[tbp]
\centering
\subfigure{\includegraphics[width=0.3\textwidth, trim={0.4cm 0cm 0.5cm 0},clip]{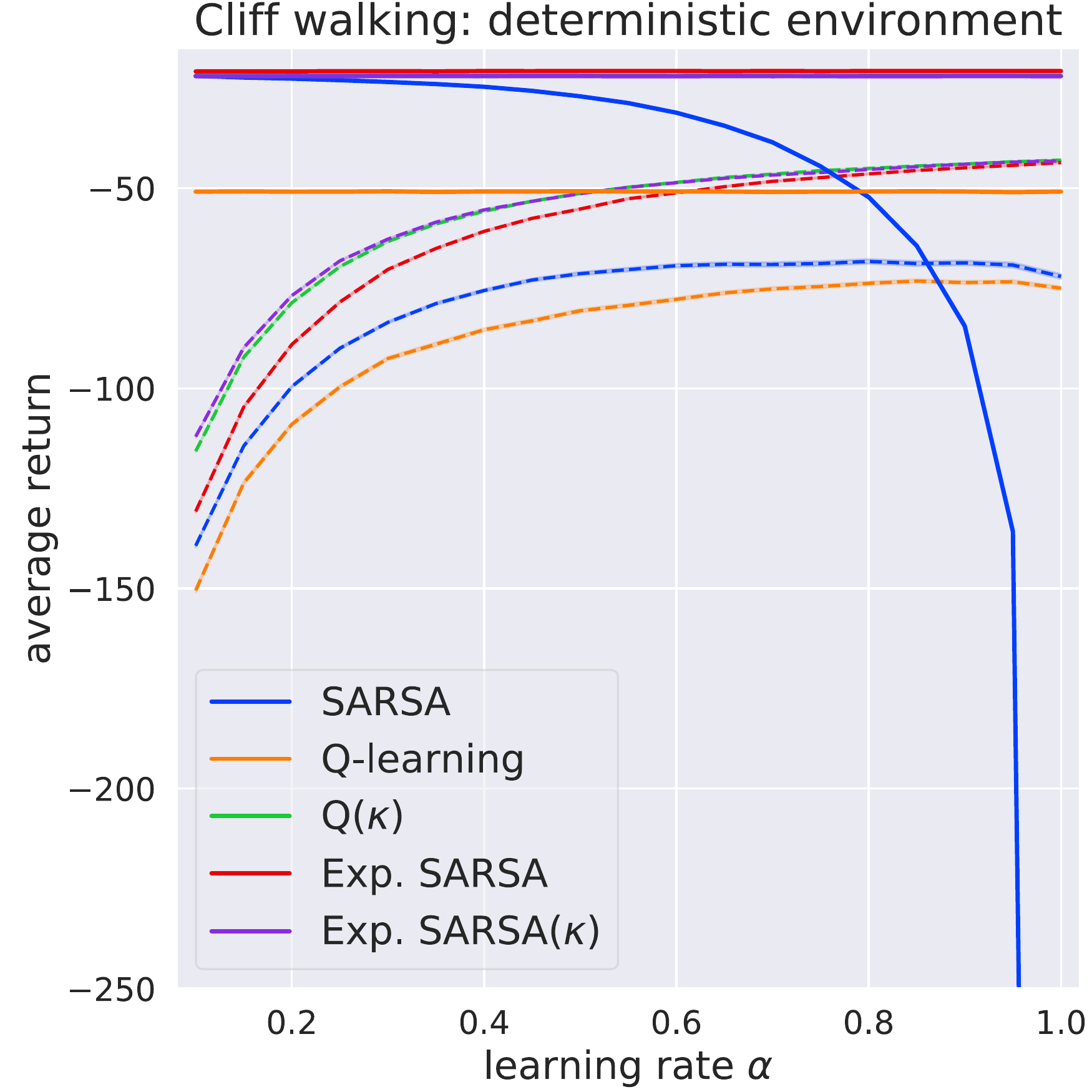}
% \label{fig:cliff_deter}
}
\hspace{1em}%
\subfigure{\includegraphics[width=0.3\textwidth, trim={0.4cm 0cm 0.5cm 0},clip]{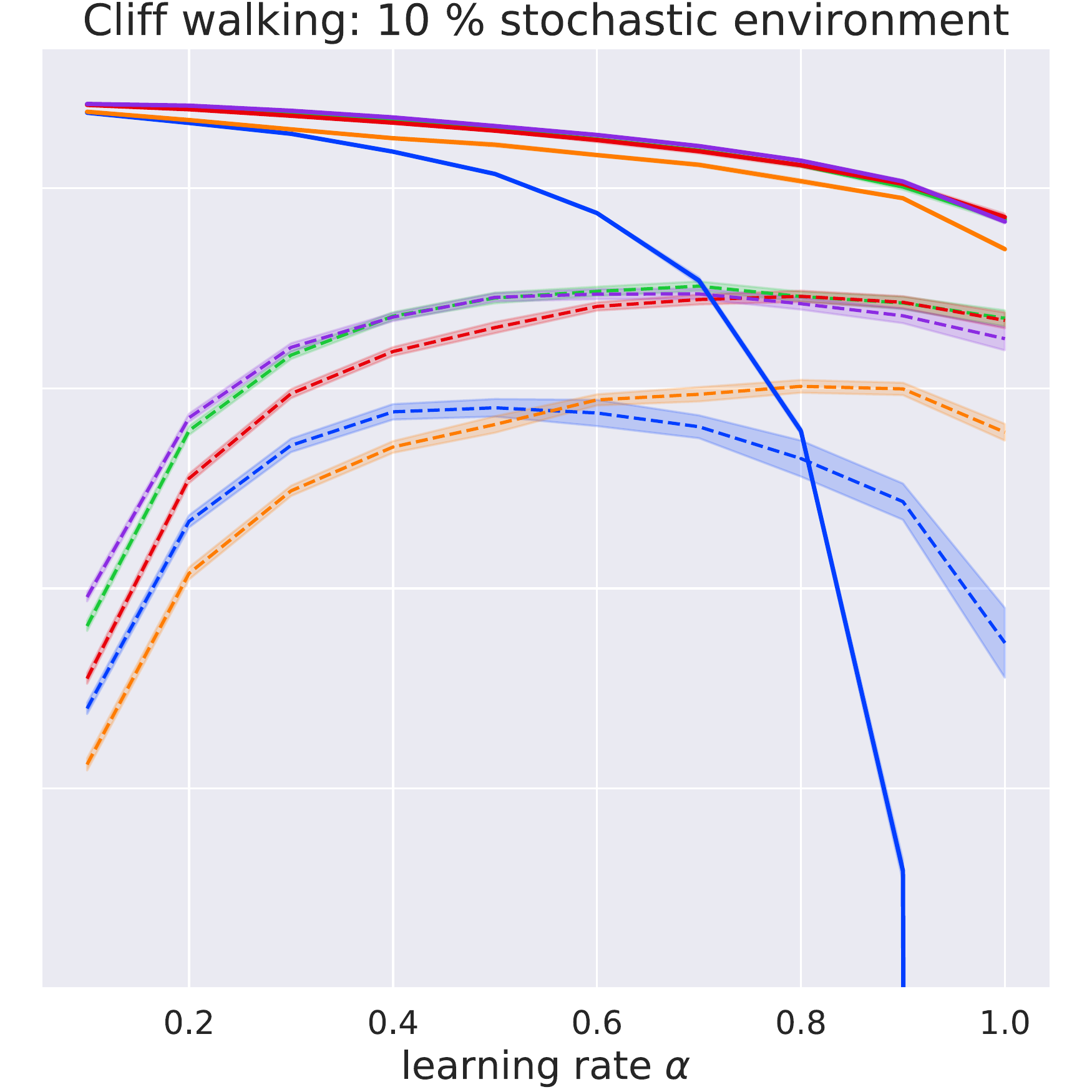}
% \label{fig:cliff_stoch}
}
\hspace{1em}%
\subfigure{\includegraphics[width=0.3\textwidth, trim={0.4cm 0cm 0.5cm 0},clip]{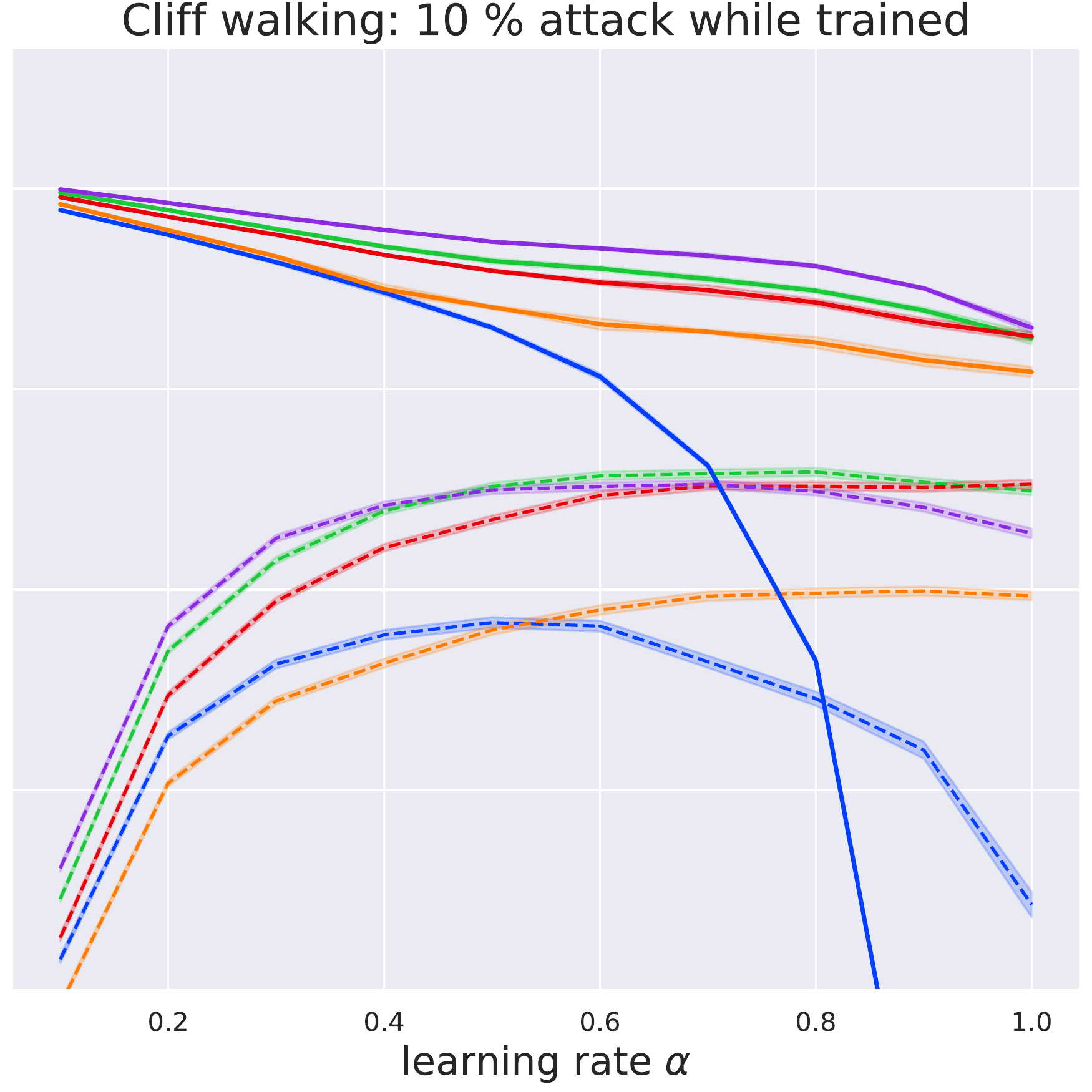}
% \label{fig:cliff_attack}
}\\
\vsp{-2em}
\subfigure{\includegraphics[width=0.3\textwidth, trim={0.1cm 0cm 0.6cm 0},clip]{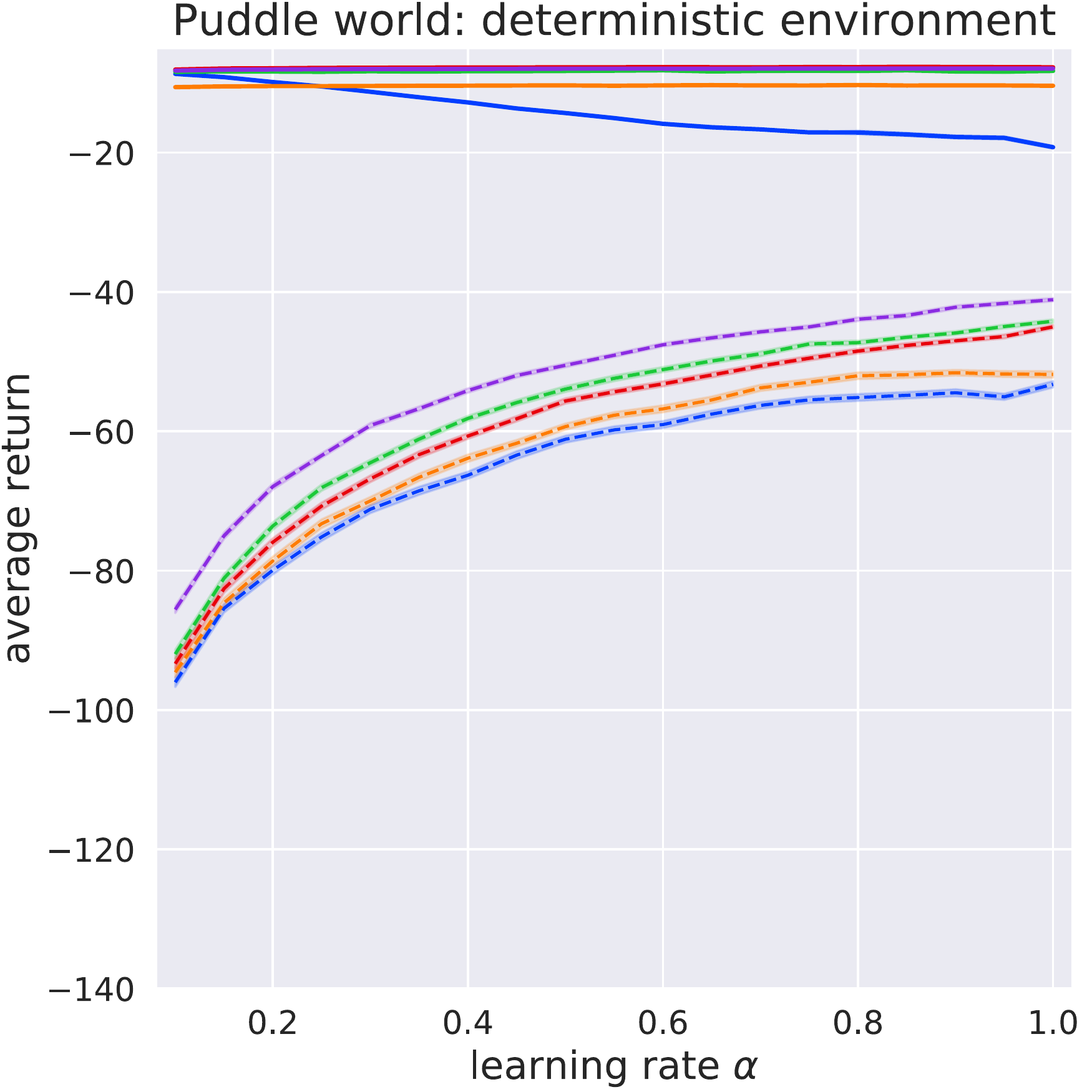}
% \label{fig:puddle_deter}
}
\hspace{1em}%
\subfigure{\includegraphics[width=0.3\textwidth, trim={0.1cm 0cm 0.6cm 0},clip]{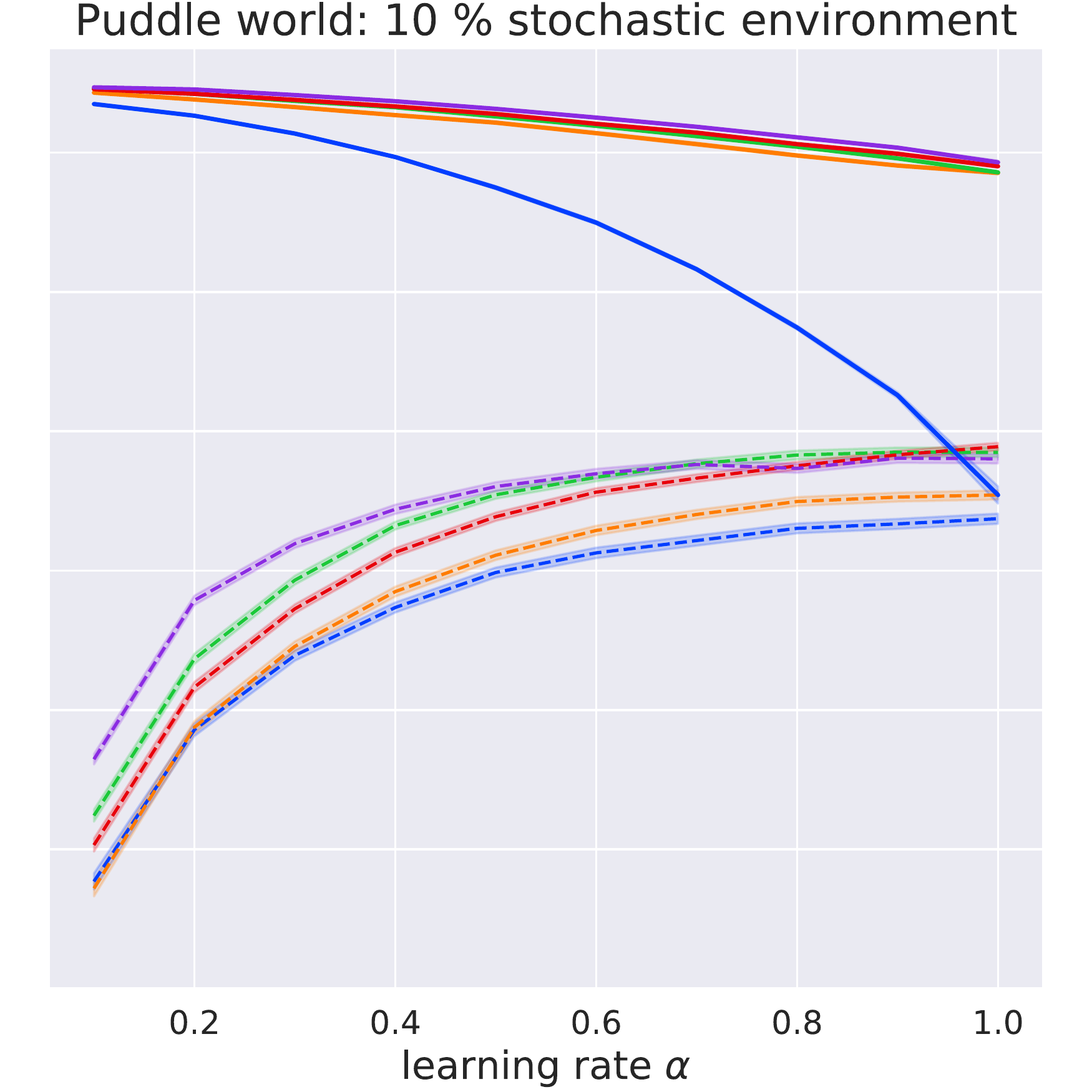}
% \label{fig:puddle_stoch}
}
\hspace{1em}%
\subfigure{\includegraphics[width=0.3\textwidth, trim={0.1cm 0cm 0.6cm 0},clip]{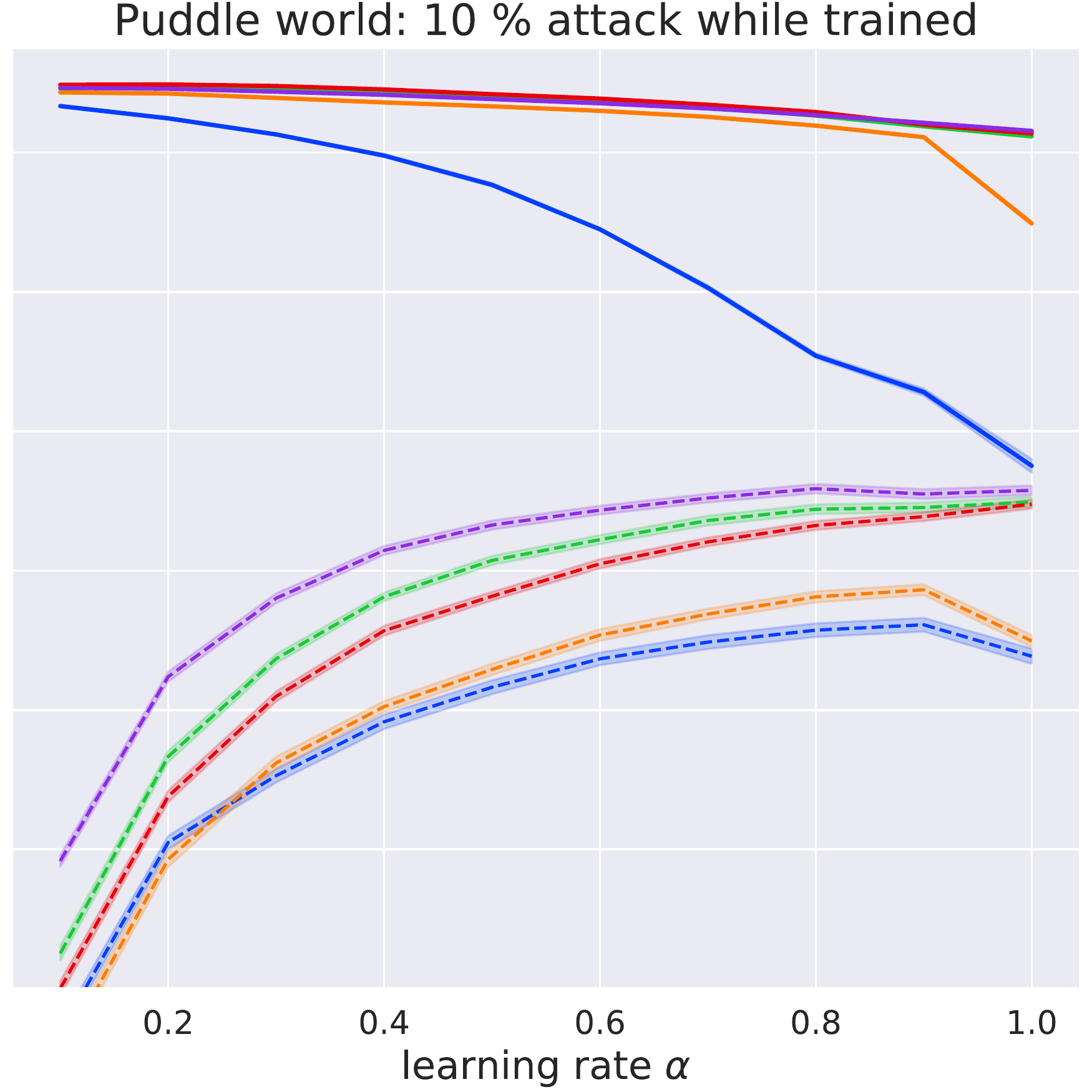}
% \label{fig:puddle_attack}
}
\vsp{-1.7em}
\caption{The Cliff Walking (single-agent) in first row and the Puddle World (multi-agent) in second row. Deterministic environment (first column), 10 \% stochastic environment (second column) and 10 \% attack while training (third column). $\epsilon$-greedy policy with fixed $\epsilon=0.1$. Early performance - dashed lines ($100$ episodes), converged performance - solid lines ($100,000$ episodes).}
\label{fig:puddle_cliff_comparing}
\vsp{-1em}
\end{figure*}
Figure~\ref{fig:puddle_cliff_comparing} shows the results for three different settings of both scenarios: (i) a deterministic environment, where each action chosen by the policy is executed with certainty; (ii) an environment with $10 \%$ stochasticity, in which a random action is taken with $10 \%$ of the time; and (iii) an environment with $10 \%$ probability of attack, in which an adversarial action is taken $10 \%$ of the time. As before, we define an attack as an action that minimizes the Q-value in the given state. 
The stochastic environment can be seen as modelling random failures.

The early performance experiments are averaged over 300 trials and the converged performance experiments are averaged over 10 trials. We also show the $95 \%$ confidence intervals on all results. We fix the exploration rate to $\epsilon = 0.1$; for the $\ko$ methods we set $\kp=0.1$ (later in this section we also experiment with different settings of $\kp$). Note that the y-axis, showing the average return, is the same in each row for easy comparison. The x-axis shows different learning rates $\alpha$. We can see how the average return decreases with more complex scenarios, from deterministic, over to stochastic, to one with attacks. The $\ko$ methods are superior to the other baselines in the early performance experiments, especially in the attack case, which is the scenario the $\ko$ methods are designed for. In the converged performance experiments the $\ko$ methods beat Q-learning and SARSA and performs at least as well as Expected SARSA.

% \lipsum

% \subsection{Early performance over episodes [TENTATIVE]}
% \begin{figure}[h]
% % \centering
% \subfigure[CLIFF: deterministic]{\includegraphics[width=0.225\textwidth, trim={0.2cm 0.6cm 0.8cm 0},clip]{figures/cliff_deter_episodes.pdf}
% \label{fig:cliff_deter_epis}
% }
% \subfigure[PUDDLE: deterministic]{\includegraphics[width=0.225\textwidth, trim={0.2cm 0.6cm 0.8cm 0},clip]{figures/puddle_deter_episodes.pdf}
% \label{fig:puddle_deter_epis}}
% \vsp{-1em}
% \subfigure[CLIFF: 10 \% stochastic env.]{\includegraphics[width=0.225\textwidth, trim={0.2cm 0.6cm 0.8cm 0},clip]{figures/cliff_stoch_episodes.pdf}
% \label{fig:cliff_stoch_epis}}
% \subfigure[PUDDLE: 10 \% stochastic env.]{\includegraphics[width=0.225\textwidth, trim={0.2cm 0.6cm 0.8cm 0},clip]{figures/puddle_stoch_episodes.pdf}
% \label{fig:puddle_stoch_epis}}
% \vsp{-1em}
% \subfigure[CLIFF: 10 \% attack while trained]{\includegraphics[width=0.225\textwidth, trim={0.2cm 0.6cm 0.8cm 0},clip]{figures/cliff_attack_episodes.pdf}
% \label{fig:cliff_attack_epis}}
% \subfigure[PUDDLE: 10 \% attack while trained]{\includegraphics[width=0.225\textwidth, trim={0.2cm 0.6cm 0.8cm 0},clip]{figures/puddle_attack_episodes.pdf}
% \label{fig:puddle_attack_epis}}
% % \vsp{-1em}
% \caption{Cliff walking (single-agent) and Puddle world (multi-agent). Early performance. [IF USED, ADD LEGEND TO FIRST PLOT]}
% \label{fig:puddle_cliff_episodes}
% \vsp{-1em}
% \end{figure}

\subsection{Different Levels of Probability of Attack}

In this section we investigate how the methods behave under different levels of attack, defined by the probability of attack per state. We consider an attack on trained (converged) methods, thus we first train each method for $100,000$ episodes (in deterministic environment) and then we test it on $50,000$ trials with given probability of attack per state. We average the results over $10$ trials and provide $95 \%$ confidence intervals. Note, that this is a different methodology of testing the methods against an adversarial attack compared to the experiments in Figure~\ref{fig:puddle_cliff_comparing}, where we considered attacks during training. This experiment shows the strength of the $\ko$ methods for different levels of attacks. We assume the probability of attack to be known here and thus we set the parameter $\kp$ to be equal to that probability, which is the meaning of the parameter $\kp$ as described before. In other words, parameter $\kp$ prescribes how much safely we want to act. We consider very rare attacks ($0.001$ probability of attack in each state) to more frequent attacks ($0.2$ probability of attack in each state) as shown in Figure~\ref{fig:probs_attack}.
For better visualisation we use logarithmic axes. We train all the methods with fixed exploration rate $\epsilon=0.1$ and learning rate $\alpha = 0.1$, note that the methods (except SARSA) converge to the same result for different learning rates as shown in left panel of Figure~\ref{fig:puddle_cliff_comparing}. SARSA is very unstable for different learning rates (demonstrated by wide confidence intervals), learns different paths for different $\alpha$ and does not converge fast enough or not at all, which can be partly explained by its higher variance~\cite{vanseijen2009}. We test the different levels of probability of attack on the Cliff Walking experiment in the left panel of Figure~\ref{fig:probs_attack}, where we can see that the $\ko$ methods compare favourably to the other baselines, however in some parts they give similar performance as Expected SARSA or SARSA. The Cliff Walking experiment has a limited expressiveness for testing the methods due to a limited number of possible safe paths with low costs (see Figure~\ref{fig:cliff_walking_map}), which is the reason for the $\ko$ methods to show only similar performance compared to the baselines, not reaching their full potential. However, the Puddle World is more expressive, because there are several possible paths differing in level of safety and cost. The bigger solution space of the Puddle World is also induced by the two cooperating agents, each having their own action space. Therefore, on the right panel of Figure~\ref{fig:probs_attack} we show the Puddle World experiment for different levels of probability of attack. Here, we can clearly see the $\ko$ methods outperform the baselines, especially Q($\ko$) is superior over the whole range of considered probabilities of attack. Note that Q($\ko$) learns a safer path even for very rare attacks ($0.001$), which is also shown in Figure~\ref{fig:puddle_world_map}, where Q($\ko$) learns a path with the same cost (distance) compared to Q-learning, but further to the puddles.

\begin{figure}[tbp]
    % \centering
    % \subfloat[aa]{
    \subfigure{
    \includegraphics[width=0.225\textwidth, trim={0 0 0 0},clip]{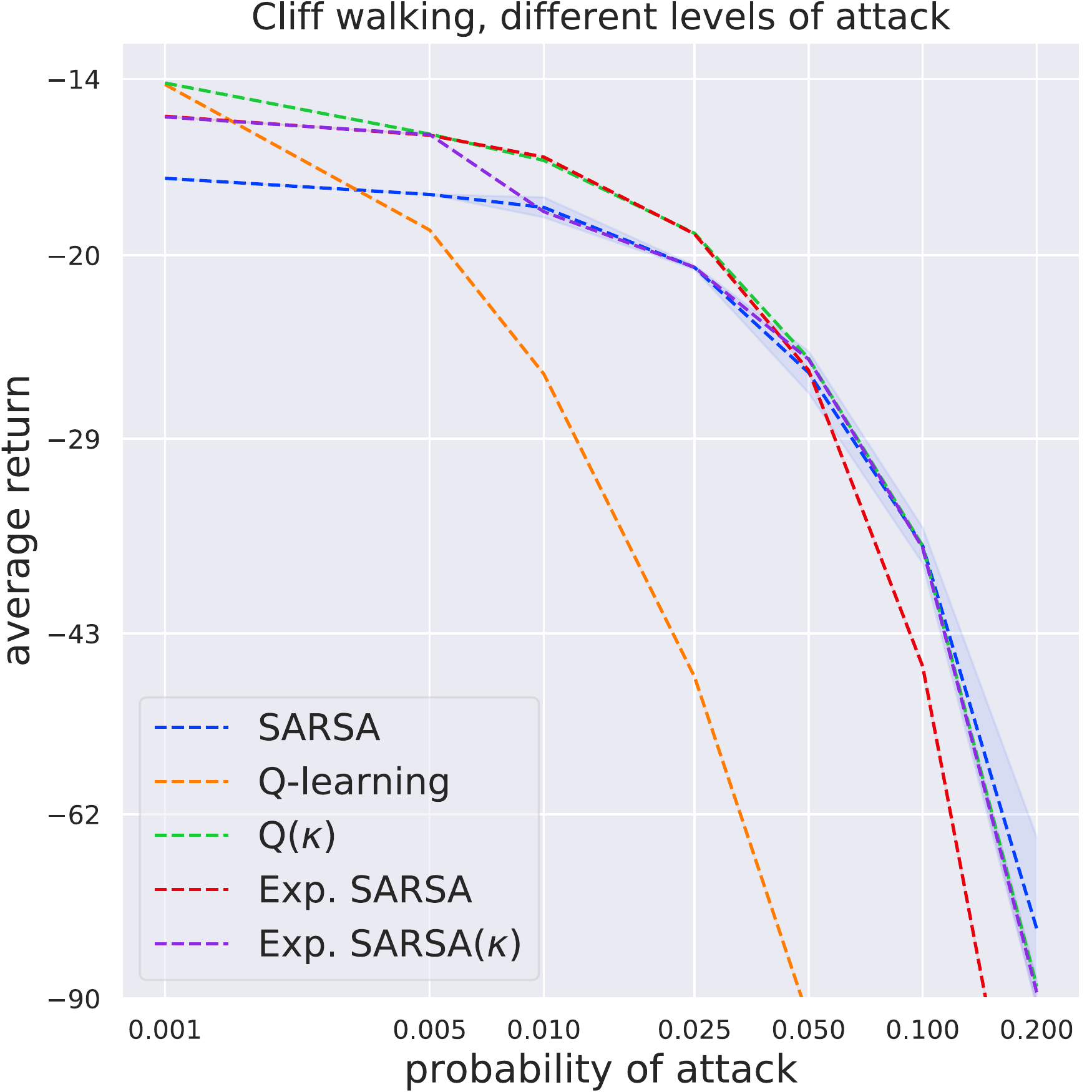}
    % \caption{aa}
    }
    % \subfloat[33]{
    \subfigure{
    \includegraphics[width=0.225\textwidth, trim={0 0 0 0},clip]{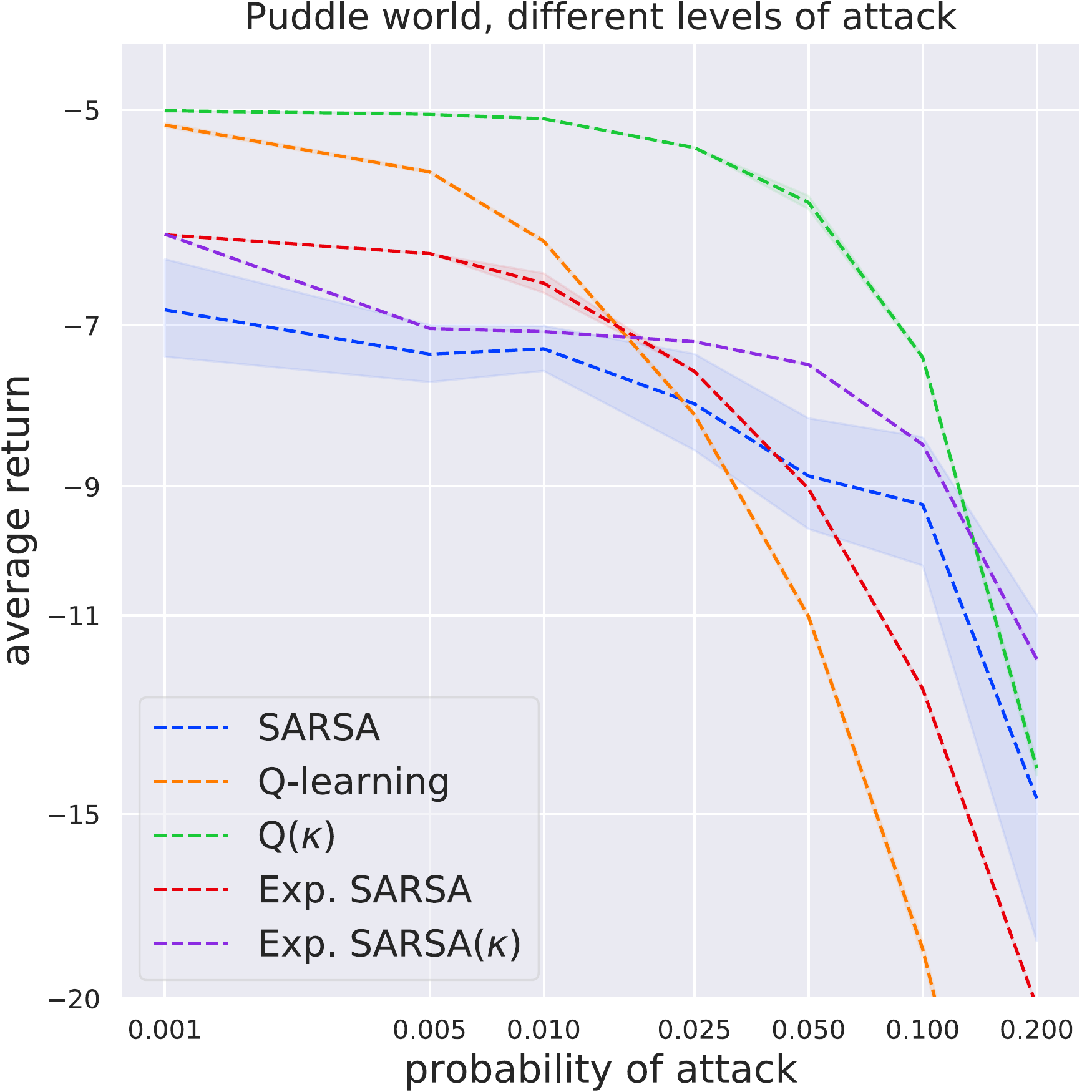}
    }
    \vsp{-1.7em}
    \caption{Varying probability of attack: Cliff Walking (left), Puddle World (right), trained 100k, test 50k, $\alpha = 0.1$, $\epsilon=0.1$.}
    \label{fig:probs_attack}
    \vsp{-1.5em}
\end{figure}

\subsection{Robustness Analysis}

We now test the robustness of the proposed algorithms to an incorrect attack model, meaning that the value of $\kp$ in Q($\ko$) and Expected SARSA($\ko$) no longer matches the actual probability of attack (in our previous experiments $\kp$ matched the actual probability of attack precisely). 
Figure~\ref{fig:robust} shows the performance of our algorithms for a range of actual attack probabilities (y-axis) while learning using a fixed parameter $\kp = 0.1$.
\begin{figure}[tbp]
    % \centering
    % \subfloat[aa]{
    \subfigure{
    \includegraphics[width=0.225\textwidth, trim={0 0 0 0},clip]{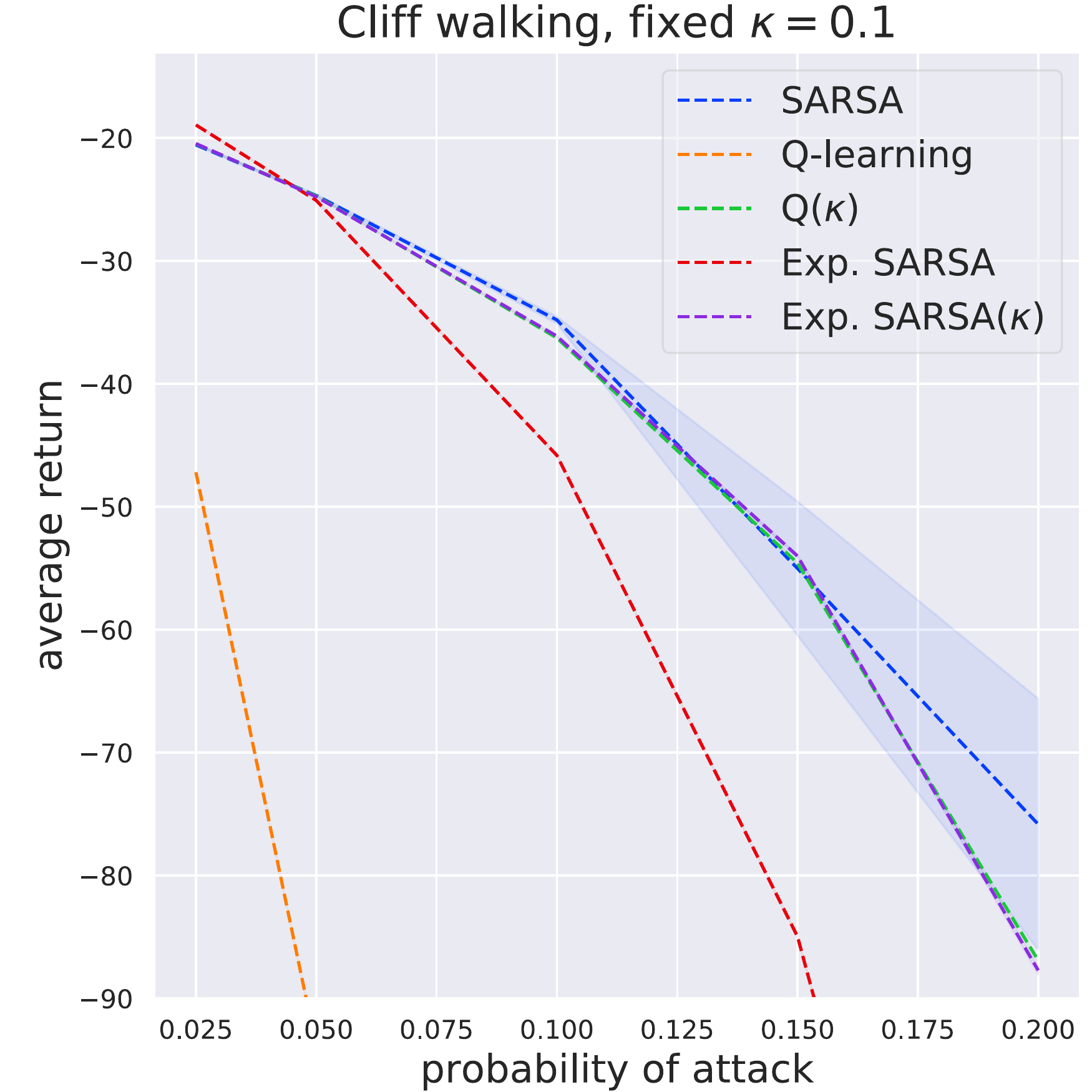}
    % \label{fig:robust_cliff}
    % \caption{aa}
    }
    % \subfloat[33]{
    \subfigure{
    \includegraphics[width=0.225\textwidth, trim={0 0 0 0},clip]{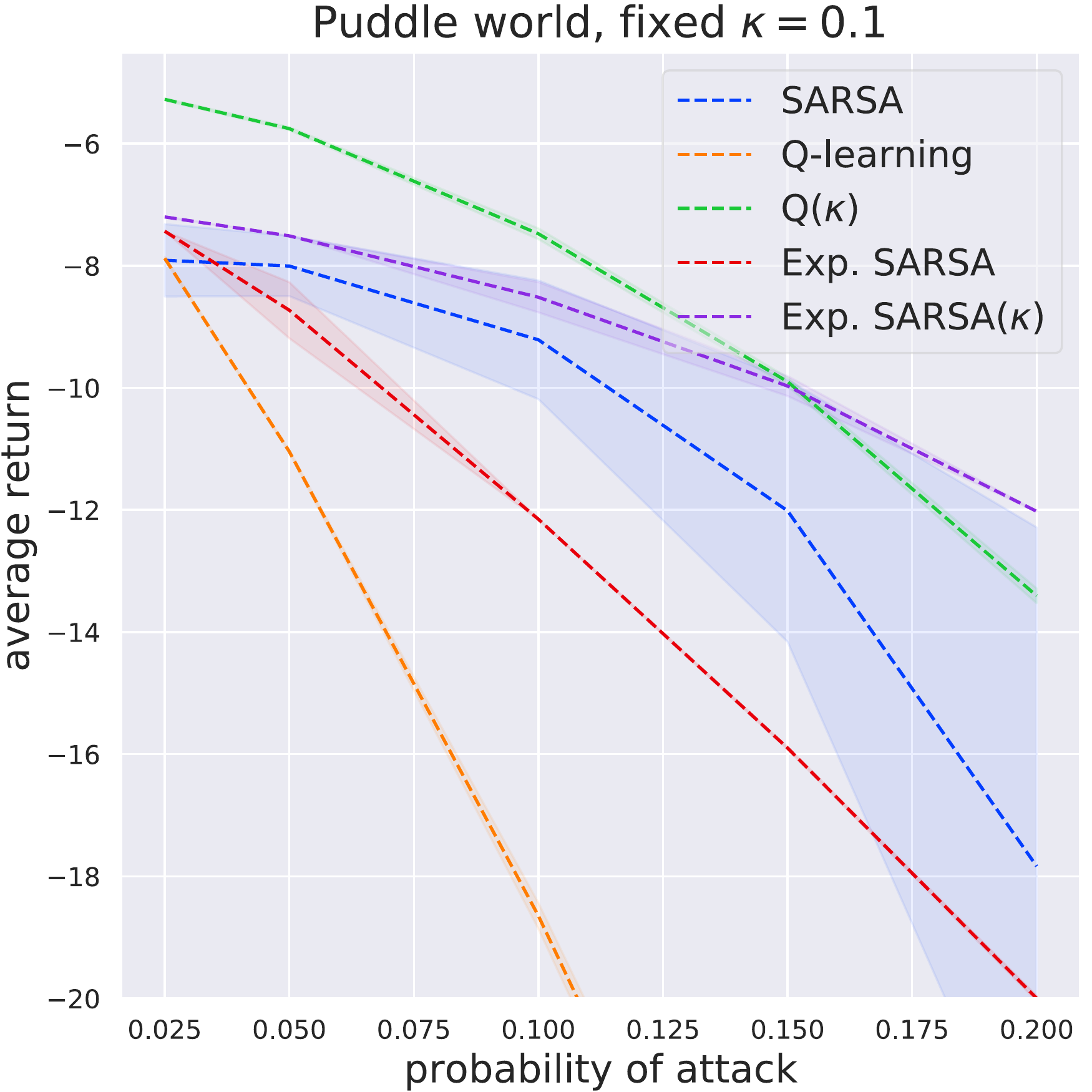}
    % \label{fig:robust_puddle}
    }
    \vsp{-1.7em}
    \caption{Robustness analysis: Cliff Walking (left), Puddle World (right), trained 100k, test 50k, $\alpha=0.1$, $\epsilon=0.1$, $\kp=0.1$.}
    \label{fig:robust}
    \vsp{-1.5em}
\end{figure}
To better highlight the robustness of our methods we choose a range of relatively high actual probabilities of attack around the fixed value of $\kp = 0.1$ (note that we no longer use a logarithmic scale). 
%Note that we no longer use logarithmic scales as in Figure~\ref{fig:probs_attack} and do not consider very rare attacks, only probabilities of attack around $0.1$.
One can see that even when $\kp$ is not equal to the actual probability of attack the proposed $\ko$ algorithms still outperform the baselines in most cases. In the Cliff Walking experiment (Figure~\ref{fig:robust} left) the $\ko$ methods perform similar to SARSA, however SARSA is quite unstable, as discussed before and as one can see by the width of the confidence interval. The Puddle World experiment (Figure~\ref{fig:robust} right) demonstrates the superior performance of the $\ko$ methods, which beat all the baselines even for the fixed parameter $\kp$. These results show that even when we do not know the probability of attack accurately we can learn a more robust strategy using the $\ko$ methods.

\section{Discussion and Conclusion}\label{sec:discussion}

We presented a new operator $\ko$ for temporal difference learning, which improves robustness of the learning process against potential attacks or perturbations in control. We proved convergence of Q($\ko$) and Expected SARSA($\ko$) to 
(i) the optimal value function $Q^\star$ of the original MDP in the limit where $\kp \rightarrow 0$; and (ii) the optimal robust value function $Q^\star_{\ko}$ of the MDP that is generalized w.r.t. $\ko$ for constant parameter $\kp$, in both single- and multi-agent versions of the methods.
% In the latter case we also proved convergence of a cooperative joint-action learning version of our methods.
%the optimal fixed point for decreasing parameter $\kp$ and under the classic GLIE (greedy in the limit with infinite exploration) assumption. Furthermore, we proved the convergence to a safe fixed point for fixed parameter $\kp$ for the both $\ko$ methods. These proofs assume a single-agent setting, and are complemented by discussing the cooperative multi-agent setting.
%
% While the extended setting of losing control could be seen just as a new Markov game, we argue that a)~this class of games is relevant and useful in applications, and b)~our formulation based on a control loss model provides a solution that exploits the identical action space of the attacker, which model free approaches applied to the extended game would not.
%
Our complementary empirical results demonstrated that the proposed $\ko$-methods indeed provide robustness against a chosen scenario of potential attacks and failures. % in both single- and multi-agent settings.
Although our method assumes that a model of such attacks and failures is known to the agent, we further demonstrated that our methods are robust against small model errors. Moreover, we have shown that even in absence of attacks or failures, our method learns a policy that is robust in general against environment stochasticity, in particular in the early stages of learning.
%, described by the probability of attack or malfunction of every part of the studied system. This control transition function is assumed to be a known robustness target set by the system designer, which provides a useful degree of freedom for many real world domains where expertise is available on the nature of expected attacks or potential malfunctions that may arise.
%
%The proposed algorithms present a new way of learning robust policies in the presence of rare malfunctions or attacks in cooperative multi-agent scenarios. 
% The preliminary experimental evaluation in a two-agent grid world domain shows promising results, outperforming standard Q-learning, SARSA and Expected SARSA algorithms in both full-communication and more daring no-communication scenarios. %A next step for future work is to investigate exactly how much information is needed for the proposed algorithm to work in multi-agent settings, ranging from none in the independent learner case, to full information in the JAL-FC case presented here. 

% \smallskip
% \note[Daan]{This still needs to be edited.}
There are several interesting directions for future work. The control space can be extended, allowing for more agents being attacked or malfunctioning with different intensity, or with control transitions depending on additional variables other than the state. %, e.g., making the control depending for example on time.
% Such extensions would narrow the reality gap and would allow for learning more complex policies, where we believe our approach could prove even more competitive.
%
Furthermore, the target of adversarial policies could be learned from experience using ideas from opponent modelling (e.g. DPIQN~\cite{hong2018}).
Our proposed operator $\ko$ can potentially be combined with some recent state-of-the-art reinforcement learning methods. For example, the operator could be combined with the multi-step Retrace($\lambda$)~\citep{munos2016safe} algorithm, potentially speeding up convergence. Mixed multi-step updates could be introduced by combination with Q($\sigma$)~\citep{deasis2017multi}, where the parameter $\sigma$ can also be state-dependent similarly to the control transitions in our model, allowing to learn robust policies against e.g. multi-step attacks. Another interesting extension along this line would be to model the control transition similar to the options framework~\citep{sutton1999options,bacon2017option}, in which case the alternate control policies could be seen as ``malicious'' options over which the agent has no control, with potentially complex initiation sets and termination conditions. Such extensions would further increase the flexibility of our proposed operator and narrow the reality gap, making it applicable to a wide range of real-world scenarios.

\begin{acks}
This project has received funding in the framework of the joint programming initiative ERA-Net Smart Energy Systems' focus initiative Smart Grids Plus, with support from the European Union's Horizon 2020 research and innovation programme under grant agreement No 646039. We are indebted to the anonymous reviewers of AAMAS 2019 for their valuable feedback.
\end{acks}

\clearpage

\bibliographystyle{ACM-Reference-Format}  % do not change this line!
\balance
\DeclareRobustCommand{\VAN}[3]{#3}
\DeclareRobustCommand{\VON}[3]{#3}
\DeclareRobustCommand{\DE}[3]{#3}
\bibliography{references}  % put name of your .bib file here

\end{document}